\documentclass{article}

\usepackage{microtype}
\usepackage{graphicx}
\usepackage{subcaption}
\usepackage{booktabs} %
\usepackage{xcolor}

\usepackage[preprint]{icml2026}

\usepackage{amsmath}
\usepackage{amssymb}
\usepackage{mathtools}
\usepackage{amsthm}

\usepackage[capitalize,noabbrev]{cleveref}

\theoremstyle{plain}
\newtheorem{theorem}{Theorem}[section]

\newtheorem{lemma}[theorem]{Lemma}

\theoremstyle{definition}

\theoremstyle{remark}

\usepackage[textsize=tiny]{todonotes}

\icmltitlerunning{Probe-Space Preconditioning for Fast and Stable Zero-Order Training}

\usepackage{booktabs}       %
\usepackage{amsfonts}       %
\usepackage{nicefrac}       %
\usepackage{microtype}      %
\usepackage{xcolor}         %
\usetikzlibrary{calc,arrows.meta}

\usepackage{graphicx}

\usepackage{pgfplots}
\pgfplotsset{compat=1.17}
\usepackage{pgfplotstable}

\newcommand{\maybeincludegraphics}[2][]{%
  \IfFileExists{#2}{\includegraphics[#1]{#2}}{\fbox{\rule{0pt}{1.2in}\rule{0.9\linewidth}{0pt}}}%
}

\newcommand{\R}{\mathbb{R}}

\begin{document}

\twocolumn[
  \icmltitle{Probe-Space Preconditioning for Fast and Stable Zero-Order Training }

  \begin{icmlauthorlist}
    \icmlauthor{Francois Chaubard}{stanford}
    \icmlauthor{Mykel J. Kochenderfer}{stanford}
    \icmlauthor{Chris R{\'e}}{stanford}
  \end{icmlauthorlist}

  \icmlaffiliation{stanford}{Stanford University, Stanford, CA, USA}

  \icmlcorrespondingauthor{Francois Chaubard}{fchaubar@stanford.edu}

  \icmlkeywords{Machine Learning, ICML}

  \vskip 0.3in
]

\printAffiliationsAndNotice{}  %

\begin{abstract}
Backpropagation (BP) dominates deep learning but imposes a massive memory tax. For example, training OPT-30B with Adam requires $\approx$ 600GB of GPU memory (assuming batch size 8 and sequence length 2048). Alternatively, zero-order optimization (ZOO) trains in inference-mode (requiring only $\approx$ 60GB for the same model): no stored activations, no gradients, and no optimizer states. However, ZOO convergence has lagged behind BP. In this work, we evaluate two methods to close this gap. First, we show that reallocating training compute budget from many steps to large effective batch sizes with many perturbations (or probes) but fewer steps, allows 1SPSA~\cite{spall1992multivariate} to outperform zero order methods like MeZO~\cite{malladi2022mezo} with less training compute. Next, we introduce \textbf{1.5-SPSA}, adding a single ``clean" forward-pass per step to 1SPSA to calculate a cheap diagonal preconditioner in probe-space, which improves convergence rate and convergence by down-weighting high curvature directions. Benchmarking on 6 post-training datasets on both Qwen3 and OPT model families, we show that 1.5-SPSA achieves State-of-the-Art results over previous ZOO solvers with much less optimization steps. For example, we train OPT-13B (for direct comparison to MeZO) and find 1.5-SPSA achieves +3.1\% accuracy on SST-2 over both MeZO and BP in only 70 steps vs. MeZO's 100,000 steps. Finally, we combine an 8-bit-packing random generator, triton fused unpack/apply kernels, and distributed parallelism to achieve fast and stable training of models as large as OPT-30B in-place on commodity GPUs (e.g. A100).
\end{abstract}

\section{Introduction}
\label{sec:introduction}

Current deep learning (DL) solvers typically rely on Backpropagation (BP) combined with adaptive optimizers like Adam~\cite{kingma2017adammethodstochasticoptimization}. While effective, this paradigm incurs a massive memory tax. For example, training OPT-30B with Adam requires $\approx$ 600GB of GPU memory (assuming batch size 8 and sequence length 2048). This prohibits training on a single accelerator for large models, forcing the use of large, expensive accelerator clusters and complex distributed model sharding.

Additionally, neural network loss landscapes typically exhibit highly ill-conditioned hessians with massive eigenvalue spreads~\cite{ghorbani2019investigation, sagun2017empirical} requiring adaptive optimizers like Adam for stable training. This is exacerbated in deep settings like Reinforcement Learning or Large Language Model (LLM) post-training ~\cite{dauphin2014saddle}.

This provides the following desiderata for a new solver: 
\begin{enumerate}
    \setlength{\itemsep}{1pt}
    \setlength{\parskip}{0pt}
    \setlength{\parsep}{0pt}
    \item \textbf{Inference-Mode Memory Use}: require no saved activations, gradients, or moving averages
    \item \textbf{Robust to Ill-Conditioned Loss}: capable of converging quickly even in ill-conditioned loss landscapes as is common in training neural networks
    \item \textbf{Efficient Training Compute}: train models with efficiency, measured as the amount of performance gain per floating-point operation
\end{enumerate}

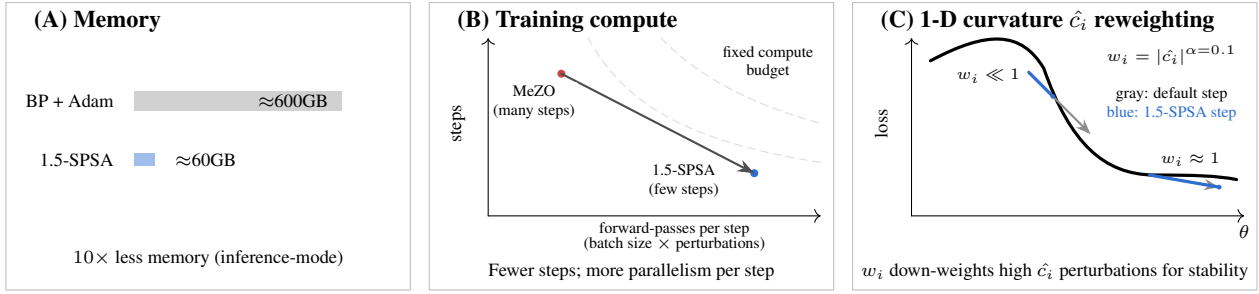
\begin{figure*}[!t]
\centering
\begin{tikzpicture}[x=1cm,y=1cm, font=\scriptsize, line cap=round, line join=round]

\definecolor{oursblue}{RGB}{45,110,210}
\definecolor{mezored}{RGB}{205,70,70}
\definecolor{softgray}{RGB}{140,140,140}

\def\W{5.35}
\def\H{3.80}
\def\G{0.25}
\pgfmathsetmacro{\Bx}{\W+\G}
\pgfmathsetmacro{\Cx}{2*(\W+\G)}

\newcommand{\ann}[3][]{%
  \node[
    fill=white,
    fill opacity=0.92,
    text opacity=1,
    inner sep=1.2pt,
    rounded corners=1pt,
    align=center,
    font=\tiny,
    #1
  ] at #2 {#3};
}

\begin{scope}[shift={(0,0)}]
  \draw[black!25] (0,0) rectangle (\W,\H);
  \node[anchor=west,font=\small\bfseries] at (0.25,\H-0.25) {(A) Memory};

  \def\xlab{1.55}
  \def\xbar{1.70}
  \def\Lbp{2.75}     %
  \def\Lzo{0.275}    %

  \def\barHbp{0.26}
  \def\barHzo{0.16}

  \def\yBP{2.50}
  \def\yZO{1.72}

  \node[anchor=east] at (\xlab,\yBP) {BP + Adam};
  \node[anchor=east] at (\xlab,\yZO) {1.5-SPSA};

  \fill[black!18] (\xbar,\yBP-\barHbp/2) rectangle ++(\Lbp,\barHbp);
  \fill[oursblue!45] (\xbar,\yZO-\barHzo/2) rectangle ++(\Lzo,\barHzo);

  \node[anchor=east] at (\xbar+\Lbp-0.06,\yBP) {$\approx$600GB};
  \node[anchor=west] at (\xbar+\Lzo+0.14,\yZO) {$\approx$60GB};

  \node[font=\scriptsize] at (\W/2,0.42) {$ 10\times$ less memory (inference-mode)};
\end{scope}

\begin{scope}[shift={(\Bx,0)}]
  \draw[black!25] (0,0) rectangle (\W,\H);
  \node[anchor=west,font=\small\bfseries] at (0.25,\H-0.25) {(B) Training compute};

  \begin{scope}[shift={(0.78,0.98)}]
    \def\PW{4.40}
    \def\PH{2.35}

    \draw[->,black] (0,0) -- (\PW,0);
    \draw[->,black] (0,0) -- (0,\PH);

    \node[font=\tiny] at ({0.56*\PW},-0.18) {forward-passes per step};
    \node[font=\tiny] at ({0.56*\PW},-0.38) {(batch size $\times$ perturbations)};
    \node[font=\scriptsize, rotate=90] at (-0.40,{0.52*\PH}) {steps};

    \foreach \k in {0.30,0.52}{
      \draw[black!14,densely dashed]
        plot[smooth,domain=\k:1.0,samples=50] ({\PW*\x},{\PH*(\k/\x)});
    }
    \ann[anchor=north east]{(\PW-0.05,\PH-0.05)}{fixed compute\\budget};

    \coordinate (mezo) at ({0.22*\PW},{0.80*\PH});
    \coordinate (ours) at ({0.80*\PW},{0.24*\PH});

    \fill[mezored] (mezo) circle (1.5pt);
    \fill[oursblue] (ours) circle (1.5pt);

    \draw[-{Stealth[length=2.2mm]}, thick, black!70] (mezo) -- (ours);

    \ann[anchor=west]{($(mezo)+(-0.95,-0.37)$)}{MeZO\\(many steps)};
    \ann[anchor=east]{($(ours)+(-0.42,-0.1)$)}{1.5-SPSA\\(few steps)};
  \end{scope}

  \node[font=\scriptsize] at (\W/2,0.24) {Fewer steps; more parallelism per step};
\end{scope}

\begin{scope}[shift={(\Cx,0)}]
  \draw[black!25] (0,0) rectangle (\W,\H);
  \node[anchor=west,font=\small\bfseries] at (0.25,\H-0.25) {(C) 1-D curvature $\hat{c_i}$ reweighting};

  \begin{scope}[shift={(0.78,0.98)}]
    \def\PW{4.40}
    \def\PH{2.35}

    \draw[->,black] (0,0) -- (\PW,0) node[below,font=\scriptsize] {$\theta$};
    \draw[->,black] (0,0) -- (0,\PH);
    \node[font=\scriptsize, rotate=90] at (-0.40,{0.52*\PH}) {loss};

    \draw[very thick]
      (0.25,2.05)
      .. controls (1.00,2.45) and (1.45,2.45) .. (1.75,1.95)
      .. controls (2.10,0.95) and (2.55,0.55) .. (3.20,0.54)
      .. controls (3.75,0.54) and (4.05,0.53) .. (4.30,0.48);

    \coordinate (highc) at (1.55,1.90);  %
    \coordinate (lowc)  at (3.15,0.54);  %

    \coordinate (high_end_full) at ($(highc)+(0.82,-0.82)$);
    \coordinate (low_end_full)  at ($(lowc)+(0.92,-0.16)$);

    \draw[-{Stealth[length=2.0mm]}, thick, softgray] (highc) -- (high_end_full);
    \draw[-{Stealth[length=2.0mm]}, thick, softgray] (lowc)  -- (low_end_full);

    \coordinate (high_end_w) at ($(highc)+(0.32,-0.32)$);
    \draw[very thick, oursblue] (highc) -- (high_end_w);
    \fill[oursblue] (high_end_w) circle (0.9pt);

    \coordinate (low_end_w) at (low_end_full);
    \draw[very thick, oursblue] (lowc) -- (low_end_w);
    \fill[oursblue] (low_end_w) circle (0.9pt);

    \ann[anchor=north east]{(\PW-0.05,\PH-0.05)}{
    $w_i=|\hat{c_i}|^{\alpha=0.1}$ \\
    \\
    gray: default step \\
    {\color{oursblue}blue: 1.5-SPSA step} };

    \ann[anchor=west]{(0.55,1.85)}{$w_i \ll 1$ };
    \ann[anchor=west]{(3.25,0.75)}{ $w_i \approx 1$};
  \end{scope}

  \node[font=\scriptsize] at (\W/2,0.24) {$w_i$ down-weights high $\hat{c_i}$ perturbations for stability};
\end{scope}

\end{tikzpicture}

\caption{
Overview of 1.5-SPSA. (A) Inference-mode training avoids optimizer state and stored activations, reducing memory relative to BP+Adam. (B) forward-passes are reallocated from many optimization steps to large (parallelizable) batch size $\times$ perturbations per step for $44\times$ less training compute to achieve SOTA test accuracy via a ZOO method (on SST-2). (C) Directional curvature estimates $\hat{c}_i$ set weights $w_i = |\hat{c_i}|^{\alpha=0.1}$: high-curvature directions cause instability so steps are down-weighted, while flat regions yield $ w_i \approx 1$, which improves stability, convergence rate, and test accuracy with inference-mode memory use.
}
\label{fig:overview_15spsa}
\end{figure*}

Derivative-Free Optimization (DFO), or Zero-Order Optimization (ZOO), is a promising area as training is done in inference-mode: no stored activations, no gradients, and no optimizer states. 1SPSA (Simultaneous Perturbation Stochastic Approximation)~\cite{spall1992multivariate}, one of the most famous ZOO methods, performs central difference approximation along a random perturbation (or probe) to estimate gradients versus calculating derivatives and trains in inference-mode, minimizing GPU memory use. Training OPT-30B with 1SPSA, for example, requires only $\approx$ 60GB, effectively reducing the memory footprint by $\approx10\times$ compared to Adam. However, 1SPSA requires many forward-passes, which comes with more compute per step. Additionally, 1SPSA introduces large perturbation noise on top of Adam's batch noise, both of which must be mitigated to ensure stable convergence which we analyze in the post-training setting in \cref{sec:empirical_analysis}. Many recent methods have attempted to solve this, most notably MeZO~\cite{malladi2022mezo} which adapts 1SPSA for LLM post-training in-place, requiring 2 forward-passes per step for 100,000 steps. In this paper, we discover two novel and independent ways to increase convergence beyond MeZO and in some cases BP, with more parallelization and less total training compute.

For our first contribution, we study where 1SPSA should spend its forward-passes: on more optimization steps, or on more computation per step. Under a fixed forward-pass budget
$F_{1\text{SPSA}} = s \times a \times 2 \times n_{\text{pert}}$,
we sweep batch size (via gradient accumulation at micro-batch $=16$) and $n_{\text{pert}}$ while holding $F_{1\text{SPSA}}=500{,}000$ on OPT-13B SST-2. \Cref{fig:spsa_stacked_heatmaps} shows that reallocating budget toward larger effective batch size and more perturbations per step yields substantially higher accuracy in only tens of steps. For example, at effective batch size $=128$ and $n_{\text{pert}}=160$, 1SPSA reaches 94.2\% in 80 steps using 205k forward-passes, versus MeZO’s 91.4\% which uses similar forward-passes and +2.8 points. However, this regime is parallelizable across perturbations and accumulation steps so 1SPSA is capable of far faster wall-clock training time. Despite strong early progress, 1SPSA can become unstable late in training as curvature varies dramatically across random directions. Adam-style diagonal preconditioning would typically mitigate this, but storing $O(d)$ moments is incompatible with our inference-mode desiderata and estimating in only tens of steps is unlikely to assist convergence. 

For our second contribution, we introduce \textbf{1.5-SPSA}, adding a single unperturbed loss evaluation per step to estimate a cheap preconditioner in probe-space, that down-weights high-curvature probes described in \cref{sec:method}. This stabilizes training and permits larger step sizes. On OPT-13B SST-2, 1.5-SPSA reaches 94.5\% in 70 steps using 179k forward-passes, outperforming MeZO by +3.1 points with less compute and exceeding BP (+2.5 points), as shown in \Cref{fig:spsa_stacked_heatmaps}. To ensure this is not just a property of OPT, we also test on Qwen3-1.7B/8B. \Cref{tab:qwen3_8b_tasks} and \Cref{tab:qwen_lrs} shows similar performance improvement over 1SPSA and BP as 1.5-SPSA expands the stable learning-rate range and improves convergence rates and convergence.

For our third contribution, we investigate the root cause of these gains in controlled stress tests, where we discover 1.5-SPSA's improvement over 1SPSA is proportional to the loss landscapes's hessian condition number $\kappa$. Curvature reweighting in high $\kappa$ losses yields up to $6\times$ fewer steps-to-zero loss and little effect in very low $\kappa$ as described in \cref{sec:experiments}. 

Finally, for our fourth contribution, we keep training in-place and fast with the following innovations. We combine an 8-bit-packing random generator, triton fused unpack/apply kernels, and distributed parallelism to minimize wall-clock per step, enabling inference-mode training on models as large as OPT-30B on commodity accelerators (e.g. A100) described in \cref{alg:dist_15spsa} and \cref{app:triton_code_bitpack}.

\begin{figure}[!t]
\centering
\begin{tikzpicture}

\def\accmin{83.0}
\def\accmax{95.0}

\def\Rlow{244}\def\Glow{176}\def\Blow{167}
\def\Rhigh{181}\def\Ghigh{229}\def\Bhigh{197}

\newcommand{\hmcell}[5]{%
  \pgfmathsetmacro{\t}{(#3-\accmin)/(\accmax-\accmin)}%
  \pgfmathsetmacro{\tt}{max(0,min(1,\t))}%
  \pgfmathtruncatemacro{\R}{\Rlow + (\Rhigh-\Rlow)*\tt}%
  \pgfmathtruncatemacro{\G}{\Glow + (\Ghigh-\Glow)*\tt}%
  \pgfmathtruncatemacro{\B}{\Blow + (\Bhigh-\Blow)*\tt}%
  \definecolor{cellcolor}{RGB}{\R,\G,\B}%
  \node[
    fill=cellcolor,
    draw=black!18,
    minimum size=1.0,
    inner sep=6pt,
    font=\tiny,
    align=center
  ] at (#1,-#2) {\shortstack{#3\%\\[-1pt](#4)\\[-1pt](#5)}};%
}

\begin{scope}[shift={(0,0)}]

\node[font=\normalsize] at (2.5, 1.) {1SPSA};
\node[font=\normalsize] at (2.5, 0.55) {Batch size};

\foreach \c/\lab in {1/16,2/128,3/512,4/1024}{
  \node[font=\normalsize] at (\c,0.0) {\lab};
}

\foreach \r/\lab in {1/20,2/40,3/160,4/640}{
  \node[font=\normalsize, anchor=east] at (0.25,-\r) {\lab};
}

\node[font=\normalsize, rotate=90] at (-0.990,-2.5) {\# perturbations};

\hmcell{1}{1}{80.5}{2512}{102k}
\hmcell{2}{1}{92.7}{45}{14k}
\hmcell{3}{1}{93.1}{40}{51k}
\hmcell{4}{1}{93.8}{35}{90k}

\hmcell{1}{2}{80.5}{1819}{146k}
\hmcell{2}{2}{92.2}{323}{207k}
\hmcell{3}{2}{93.6}{79}{202k}
\hmcell{4}{2}{93.1}{16}{82k}

\hmcell{1}{3}{88.3}{210}{67k} 
\textbf{\hmcell{2}{3}{94.2}{80}{205k}}
\hmcell{3}{3}{92.2}{8}{82k}
\hmcell{4}{3}{93.1}{7}{143k}

\hmcell{1}{4}{89.9}{33}{42k}
\hmcell{2}{4}{92.0}{19}{195k}
\hmcell{3}{4}{93.1}{7}{287k}
\hmcell{4}{4}{93.1}{6}{492k}

\begin{scope}[shift={(5.055,-4.25)}]
  \def\cbw{0.22}
  \def\cbh{3.55}

  \foreach \i in {0,...,59}{
    \pgfmathsetmacro{\tt}{\i/59}%
    \pgfmathtruncatemacro{\R}{\Rlow + (\Rhigh-\Rlow)*\tt}%
    \pgfmathtruncatemacro{\G}{\Glow + (\Ghigh-\Glow)*\tt}%
    \pgfmathtruncatemacro{\B}{\Blow + (\Bhigh-\Blow)*\tt}%
    \pgfmathsetmacro{\yy}{\tt*\cbh}%
    \definecolor{cbcolor}{RGB}{\R,\G,\B}%
    \fill[cbcolor] (0,\yy) rectangle (\cbw,\yy+\cbh/60);
  }

  \draw[black!30] (0,0) rectangle (\cbw,\cbh);

  \foreach \v in {83,86,89,92,95}{
    \pgfmathsetmacro{\yy}{(\v-\accmin)/(\accmax-\accmin)*\cbh}
    \draw[black!45] (\cbw,\yy) -- ++(0.12,0);
    \node[font=\scriptsize, anchor=west] at (\cbw+0.16,\yy) {\v};
  }
  \node[font=\scriptsize, rotate=90] at (1,2) {SST-2 Test Acc (\%)};
\end{scope}

\end{scope}

\begin{scope}[shift={(0,-6.35)}]

\node[font=\normalsize] at (2.5, 1.) {1.5-SPSA};
\node[font=\normalsize] at (2.5, 0.55) {Batch size};

\foreach \c/\lab in {1/16,2/128,3/512,4/1024}{
  \node[font=\normalsize] at (\c,0.0) {\lab};
}

\foreach \r/\lab in {1/20,2/40,3/160,4/640}{
  \node[font=\normalsize, anchor=east] at (0.25,-\r) {\lab};
}

\node[font=\normalsize, rotate=90] at (-0.990,-2.5) {\# perturbations};

\hmcell{1}{1}{85.6}{2132}{85k}
\hmcell{2}{1}{93.1}{38}{12k}
\hmcell{3}{1}{93.1}{30}{38k}
\hmcell{4}{1}{93.1}{32}{82k}

\hmcell{1}{2}{88.3}{1099}{88k}   
\hmcell{2}{2}{93.3}{80}{51k}
\hmcell{3}{2}{94.5}{70}{179k}
\hmcell{4}{2}{93.8}{12}{61k}

\hmcell{1}{3}{93.1}{120}{38k}     
\textbf{\hmcell{2}{3}{94.5}{70}{179k}}
\hmcell{3}{3}{92.2}{29}{297k}
\hmcell{4}{3}{93.1}{7}{143k}

\hmcell{1}{4}{92.1}{30}{38k}     
\hmcell{2}{4}{92.4}{25}{256k}    
\hmcell{3}{4}{92.9}{6}{246k}
\hmcell{4}{4}{93.1}{6}{492k}

\begin{scope}[shift={(5.055,-4.25)}]
  \def\cbw{0.22}
  \def\cbh{3.55}

  \foreach \i in {0,...,59}{
    \pgfmathsetmacro{\tt}{\i/59}%
    \pgfmathtruncatemacro{\R}{\Rlow + (\Rhigh-\Rlow)*\tt}%
    \pgfmathtruncatemacro{\G}{\Glow + (\Ghigh-\Glow)*\tt}%
    \pgfmathtruncatemacro{\B}{\Blow + (\Bhigh-\Blow)*\tt}%
    \pgfmathsetmacro{\yy}{\tt*\cbh}%
    \definecolor{cbcolor}{RGB}{\R,\G,\B}%
    \fill[cbcolor] (0,\yy) rectangle (\cbw,\yy+\cbh/60);
  }

  \draw[black!30] (0,0) rectangle (\cbw,\cbh);

  \foreach \v in {83,86,89,92,95}{
    \pgfmathsetmacro{\yy}{(\v-\accmin)/(\accmax-\accmin)*\cbh}
    \draw[black!45] (\cbw,\yy) -- ++(0.12,0);
    \node[font=\scriptsize, anchor=west] at (\cbw+0.16,\yy) {\v};
  }
  \node[font=\scriptsize, rotate=90] at (1,2) {SST-2 Test Acc (\%)};
\end{scope}

\end{scope}

\end{tikzpicture}

\caption{
Comparison of 1SPSA (top) and 1.5-SPSA (bottom) for OPT-13B sweeping over $n_{pert}$ and batch size per optimization step. Each cell reports: test accuracy on SST-2 (top), optimization steps to convergence (middle), forward-passes to convergence (bottom). Convergence is defined to be when validation loss has plateaued. All configurations are compared at equal total forward-pass budgets (500,000 forward-passes), although convergence may happen before the final optimization step. We find more batch size and perturbations per step seems to increase convergence rate, but only improves final convergence value
up to a point of diminishing returns. We also find that 1.5-SPSA outperforms 1SPSA in both final convergence and convergence rate. 
}
\label{fig:spsa_stacked_heatmaps}
\end{figure}
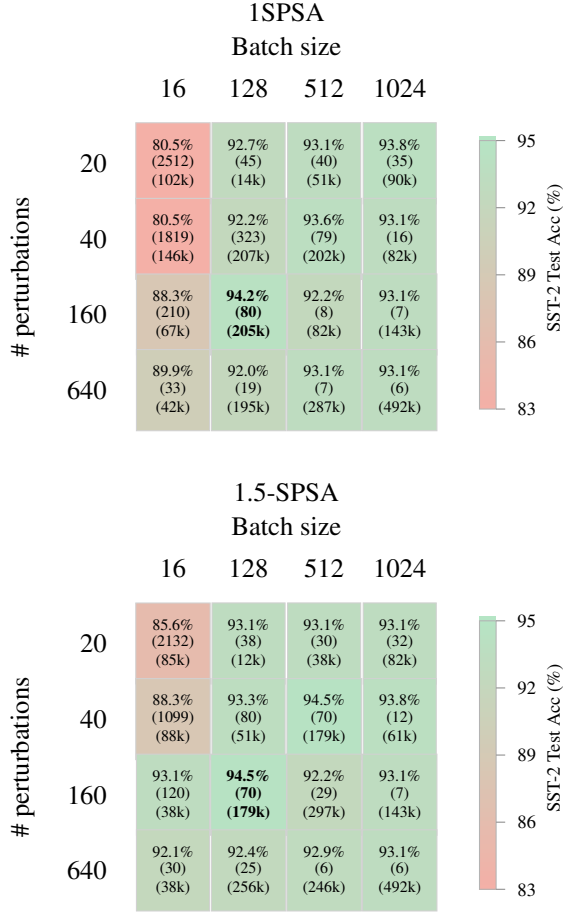

\section{Background: Stochastic Approximation and Finite Differences}

Stochastic Approximation (SA) finds roots of noisy functions. Robbins-Monro (1951) introduced the field for noisy gradients, proposing the iterative update:
\begin{equation}
    \theta_{k+1} = \theta_k - \lambda \hat{g}(\theta_k)
\end{equation}
where $\hat{g}$ is a noisy gradient estimate for model weights $\theta \in \R^d$ for loss function $L(\theta)$ and a step size $\lambda$. Kiefer-Wolfowitz (1952) extended this to \textit{noisy function values} using component-wise finite differences. However, the Kiefer-Wolfowitz estimator requires $2d$ forward-passes for a $d$-dimensional gradient, which is prohibitively expensive for deep learning in large models (e.g. $d \sim 10^9$).

To overcome the dimensionality curse, Spall (1992) introduced \textbf{Simultaneous Perturbation Stochastic Approximation (SPSA)}. SPSA approximates the gradient using a random perturbation vector $z \in \mathbb{R}^d$ and only two function evaluations (forward-passes), independent of dimension $d$ to perform central difference approximation. 

\textbf{1SPSA}. The standard gradient estimator for 1SPSA is:
\begin{equation}
\hat{g}(\theta) = \frac{1}{2n_{pert} } \sum_{i=1}^{n_{pert}}\frac{L(\theta + \epsilon z_i) - L(\theta - \epsilon z_i)}{2\epsilon} z_i^{-1}
\end{equation}
for $\epsilon>0$ and $n_{pert}$ perturbations, where $z_i^{-1}$ is the coordinate-wise inverse. This is typically sampled from the \textbf{Rademacher distribution} ($z_i \in \{-1, +1\}$), so $z^{-1} = z$, simplifying implementation, minimizing memory, and speeding up convergence to the true gradient ~\cite{spall1992multivariate}.

\textbf{2SPSA}. Spall (1997) later introduced 2SPSA to estimate a Hessian preconditioner~\cite{spall1997accelerated} in parameter space. However, estimating a $d \times d$ Hessian is impossible for modern LLMs. Standard optimizers like Adam estimate a diagonal preconditioner, which is more tractable, but still requires $O(d)$ memory~\cite{kingma2017adammethodstochasticoptimization}. We seek a solver that does not require additional $O(d)$ memory use.

\subsection{Modern ZOO for neural networks}

The main goal of modern ZOO methods adapted to train large neural networks is to reduce both (i) minibatch noise and (ii) perturbation/estimator noise, while keeping memory overhead reasonable. The way they do so differs. We breakdown the main techniques used in modern ZOO methods below.

\paragraph{Structured perturbations for LLMs.}
A practical obstacle in large models is that a single global perturbation can mix parameters with very different scales across layers. A MeZO variant called LeZO~\citep{wang2024simultaneous} addresses this in LLM post-training by applying perturbations in a structured way (layerwise), improving stability at a fixed query budget. Related work explores other structured sparsity strategies for scaling ZO updates in large models like DeepZero~\citep{chen2024deepzero}.

\paragraph{Adaptive moment methods.}
Inspired by first-order optimizer structure, these approaches maintain momentum and/or adaptive per-coordinate learning rates based on gradient estimates. For example, ZO-AdaMM~\citep{chen2019zo_adamm} provides convergence guarantees with a dimension-dependent slowdown typical of ZOO estimators. These methods can improve practical convergence, but they generally re-introduce optimizer-state memory tax that our in-place setting aims to avoid.

\paragraph{Learned subspaces.}
To improve sample efficiency, several methods reduce the perturbation search space by restricting search to a much smaller subspace, or by adapting the sampling distribution of perturbation directions over time. ASEBO learns an ``active'' subspace for evolution-strategy style gradient estimates~\citep{cottrell2020asebo}, while RSVP proposes a variance-reduced ZO scheme~\citep{gautam2024variancereducedzerothordermethodsfinetuning}. In a complementary theory line, ZOO variance-reduction methods build on SVRG/SPIDER-style ideas and provide improved query complexity for finding approximate stationary points in nonconvex problems.

\paragraph{Evolution strategies (ES).}
ES methods estimate gradients of a smoothed objective from a population of perturbed parameters and are attractive for highly parallel systems. Classic families include NES~\citep{wierstra2008nes} and CEM~\citep{RubinsteinKroese2004CrossEntropy}; modern large-scale demonstrations show strong distributed scaling properties~\citep{salimans2017evolution}.

Our method attempts to maintain the simplicity of 1SPSA exploring how far we can go with very slight tweaks to the original algorithm. 

\section{Motivation}
\label{sec:motivation}

\subsection{Training Compute Budget Analysis for 1SPSA}
\label{sec:compute_allocation}
Our first investigation into 1SPSA is to understand the impact of the choice of batch size and number of perturbations per optimization step to determine what is the most efficient use of training compute. We measure training compute as the total number of forward-passes which scales with accumulation steps $a$, number of perturbations $n_{pert}$, and optimization steps $s$. The total number of forward-passes then is calculated as $F_{1SPSA} = s \times a \times 2 \times n_{pert}$. We use matched forward-pass budgets when comparing batch size and perturbation allocations to fairly compare configurations. Many ZOO baselines allocate budget toward large $s$ with small $a$ and small $n_{pert}$. 

For our first observation, we find that budget allocations with small $s$ and large $(a,n_{pert})$ can reach much stronger accuracy in only \emph{tens of optimization steps}. While it is intuitive that higher batch size and more perturbations would result in more stable training, it is unexpected that it would outperform final convergence for both MeZO and BP in such few optimization steps.  

\Cref{fig:spsa_stacked_heatmaps} shows sweeps over batch and perturbations for SST-2 on OPT-13B. In 80 iters, 1SPSA is able to beat MeZO (94.5\% vs. 91.4\% SST-2 test accuracy) at batch size 128 and 160 perturbations per step, which at micro batch size 16 is 205k forward-passes vs. 200k for MeZO. We attribute this increase in accuracy to the large learning rate $\lambda$ made possible with stable loss landscape measurements. Note, the optimal $\lambda$ found per run is orders of magnitude higher than the typical $\lambda$ used in post-training for LLMs. We find $\lambda=5e^{-4}$ to be stable at batch size 128 and 160 perturbations, while MeZO and BP train with $\lambda=10^{-6}$, allowing our step size to be 500$\times$ larger. We observe this while training Qwen3 as well as shown in \cref{tab:qwen_lrs}.

For our second observation, we find that tying $\lambda=\epsilon$ results in the most stable training as reported in \cref{tab:eps_ablation}. This is intuitive since $\epsilon$ measures the loss landscape $L(\theta)$ at a specific radius from $\theta_k$, and $\lambda$ determines how far to step $\theta_{k+1} = \theta_{k} - \lambda \hat{g}(\theta_k,\epsilon)$. If we set $\lambda>\epsilon$, we step farther than we have measured, and would potentially be unstable. If we set $\lambda<\epsilon$, we step shorter than we have measured, and would slow convergence in stable loss landscapes or would be unstable in ill-conditioned loss landscapes. Additionally, this is more numerically-stable as the terms cancel in the update rule.

For our third observations, we observe that more batch size and perturbations per step seems to continually increase convergence rate, but only improves final convergence value up to a point of diminishing returns. Surprisingly, the best test accuracy is not at the largest batch size (1024) and $n_{pert}$ (640) tested, but at batch size (128) and $n_{pert}$ (160). We attribute this to the lack of momentum to get out of local minima. As we have no momentum, we need something to shake out of local minima. Leaving some variance may be the key to allow this to happen. 

For our fourth observation, we find that training is unstable at the end of training. Training loss diverges at the end of training and does not recover. While we use a plateau-based learning rate schedule that cuts $\lambda$ and $\epsilon$ both by half after a lack of progress, this does not remedy the issue and admittedly has little to no effect on final performance. To understand the loss landscape we are optimizing, we plot an eight thousand perturbation histogram of our 3-point curvature estimate in \cref{fig:qwen3_curvature_histogram}. We observe that curvature can swing from $-40^8$ to $30^8$ and everywhere in between in the same optimization step. If we step in each direction with the same $\lambda$, we will be stepping too little in low curvature dimensions, and too far in high curvature dimensions. With this in mind, we seek a way to down-weight these high curvature dimensions to keep training stable.

\subsection{Incorporating Curvature into 1SPSA}
In convex optimization, preconditioning with the inverse Hessian $H^{-1}$ (Newton's method) corrects for ill-conditioning. In deep learning, a full $H$ is often unavailable. First-order methods like Adam approximate diagonal curvature in parameter space. Standard 2SPSA attempts to estimate global Hessian-vector products which is intractable and violates our memory constraint in our desiderata. While we cannot afford such methods that use $O(d)$ additionally memory for the solver, especially in such few optimization steps, perhaps there exists a curvature scheme that can be effective. Since we update along random directions vectors $z_i$, we can take insight from the Johnson-Lindenstrauss (JL) lemma~\cite{JohnsonLindenstrauss1984}, which suggests that geometry in high-dimensional space is preserved in random low-dimensional projections with some probability. In this way, we are randomly projecting our optimization problem onto a random subspace spanned by our perturbations. Since we can not practically precondition in the high-dimensional space, perhaps we can precondition the low-dimensional subspace and still improve convergence. We derive in section \cref{app:jl_curvature} an extension of JL showing a key result: if JL preserves geometry on a random perturbation set, then the same perturbation set also preserves the curvature terms (up to controlled distortions), justifying curvature reweighting as a stable perturbation-space preconditioner. 

With this in mind, we propose a simpler, cheaper curvature scheme: estimate the scalar curvature only along the perturbation directions $z_i$. We can achieve this with just one additional forward-pass per optimization step (shared across all perturbations), the clean forward-pass ($L(\theta)$), which sits in the middle of all of our central difference approximations. This evaluation is actually required already to track true training loss so this may be viewed as no additional compute. With this center point, we estimate a 3-point scalar curvature term $\hat{c_i}$ per perturbation direction $z_i$ from {$L(\theta - \epsilon z_i)$,$L(\theta)$, $L(\theta + \epsilon z_i)$}.

However, as neural loss landscapes are non-quadratic and heavy-tailed~\cite{ghorbani2019investigation}, a direct inverse-curvature step would be unstable if curvature is near zero (exploding step) or extremely large (vanishing step). By clipping the curvature estimate to a safe range, we maintain stability while exploiting local geometry. We desire a sub-linear response to extreme curvature: if curvature doubles, we do not necessarily want the step size to halve, as that might be too aggressive. This motivates a $\alpha$-saturated weighting scheme described below.

\section{Method: 1.5-SPSA}
\label{sec:method}
Building on this motivation, we present 1.5-SPSA. We combine the compute allocation budget found in our 1SPSA analysis with a cheap, robust curvature-informed preconditioner in perturbation space to improve convergence rate in ill-conditioned loss landscapes as per our desiderata. Critically, we maintain 1SPSA's inference-mode memory use as well. 

\subsection{Directional Curvature and $\alpha$-Saturation}
To implement this, we modify the 1SPSA estimator to include one clean loss evaluation $L(\theta)$ per step (shared across all perturbations). For each perturbation triplet using the standard central difference stencil~\cite{spall1992multivariate} $\{L(\theta+\epsilon z_i), L(\theta), L(\theta-\epsilon z_i)\}$, we can simultaneously estimate the gradient signal \textit{and} the scalar directional curvature:
\begin{equation}
\hat{c}_i = \frac{L(\theta+\epsilon z_i) - 2L(\theta) + L(\theta-\epsilon z_i)}{\epsilon^2} \approx z_i^\top \nabla^2 L(\theta) z_i
\end{equation}

However, the Hessian spectrum in neural networks is known to be heavy-tailed~\cite{ghorbani2019investigation,sagun2017empirical}. Random perturbations occasionally align with extremely sharp directions (high $\hat{c}_i$) or extremely flat ones ($\hat{c}_i \approx 0$). A naive Newton step $\propto 1/\hat{c}_i$ would be unstable~\cite{dauphin2014saddle}. To robustly handle this, we winsorize the curvature using an \textbf{$\alpha$-saturated weighting scheme}. We define a robust weight $w_i$:
\begin{equation}
w_i = \frac{1}{\max\left( \lambda_{\text{reg}},\, |\hat{c}_i|^\alpha \right)}
\end{equation}
where $\lambda_{\text{reg}}$ is a scalar regularization term (we use $\lambda_{\text{reg}}=1$), and $\alpha \in [0, 1]$ controls the aggressiveness of the preconditioning. $\alpha=1$ corresponds to a full diagonal Newton step in perturbation space, while $\alpha=0$ recovers standard 1SPSA. and we test $\alpha \in [10^{-5},10^{-3},10^{-1},0.5,1]$ and find $\alpha=0.1$ to be consistently optimal across architectures (OPT and Qwen), tasks and model sizes as show in \cref{app:alpha_sweep}.

\subsection{1.5-SPSA Update Rule}
Substituting these definitions back into 1SPSA, we arrive at the simple 1.5-SPSA update rule:
\begin{equation}
\Delta \theta = \frac{1}{2n} \sum_{i=1}^n \left( \frac{ L(\theta + \epsilon z_i) - L(\theta - \epsilon z_i) }{ \max\left( \lambda_{\text{reg}},\, |\hat{c}_i|^\alpha \right) } \right) z_i
\end{equation}

As discussed in \cref{sec:motivation}, we tie $\lambda=\epsilon$ so those terms cancel. 

The 1.5-SPSA update can also be written compactly as
\begin{equation}
\Delta \theta \;=\; -\eta \, Z\,W\,\Delta \ell,
\end{equation}
where 
\[\eta := \frac{1}{2n_{pert}},\]
\[
Z := [z_1\;\; z_2\;\; \cdots\;\; z_{n_{pert}} ],
\]
\[
W := \mathrm{diag}(w_1,\dots,w_{n_{pert}}),
\]
\[
\Delta \ell := (\Delta \ell_1,\dots,\Delta \ell_{n_{\text{pert}}})^\top,
\quad
\Delta \ell_i := L(\theta+\epsilon z_i) - L(\theta-\epsilon z_i).
\]

\subsection{System Design to Maintain our Desiderata}
To compute this update, we first observe that all of the loss values {$L(\theta), L(\theta + \epsilon z_i), L(\theta - \epsilon z_i)$} can be calculated in parallel allowing us to achieve wall-clock time per update theoretically faster than BP (as we have no backward pass). The wall-clock time per update is then limited only by the number of accelerators available in the cluster and the speed of the forward-pass. However, $z_i$ must be expressed three times so we have a compute vs. memory tradeoff to consider. Instantiating an $O(d)$ sized random perturbation from a seed can take considerable wall-clock time which can be reduced by caching $z_i$. However this would violate our inference-only memory use. We must re-instantiate $z_i$ three times as fast as possible. So we combine three innovations to maintain our desiderata:
\begin{itemize}
    \setlength{\itemsep}{0pt}
    \setlength{\parskip}{0pt}
    \setlength{\parsep}{0pt}
    \item \textbf{Bit-Packed perturbations}: Since $z_i$ are Rademacher, we instantiate them as 1 bit per parameter and only one layer block at a time (block-size $B$ configurable). For a 13B model with $B=1$, a perturbation is $\approx 1.6$ GB (unpacked bf16) but only $\approx 1.6$ GB / 16 = 100 MB packed.
    \item \textbf{Fused Kernels}: We implement custom fused CUDA/Triton kernels that unpack the perturbation bit, scale it by $\epsilon$ or the computed weight $w_i$, and add it to the weights in-place. This combined with bit-packing provides a 2.76$\times$ speedup over pytorch implementation as described in \cref{app:triton_code_bitpack}.
    \item \textbf{Seed-Based Distribution}: We use distributed training to parallelize the forward passes as outlined in \cref{alg:dist_15spsa}. We start by broadcasting seed scalars for each perturbation to the ranks. Each GPU generates its perturbation on the fly, calculates the forward-pass, and communicates back only a scalar loss minimizing communication overhead to only scalars. Rank 0 must then regenerate the perturbation, scale it appropriately and update the model. Once the update is completed, we need to communicate the updated model $\theta$ to all GPUs. 
\end{itemize}

\begin{algorithm}[!t]
\caption{1.5-SPSA\_STEP$(\theta,\epsilon,a,n,\alpha,\lambda_{\mathrm{reg}})$\label{alg:dist_15spsa}}
\begin{algorithmic}[1]
\STATE \textbf{Input:} $\theta\in\R^d$, $\epsilon>0$, accumulation $a$, $n_{pert}$, $\alpha$, $\lambda_{\mathrm{reg}}$
\STATE \textbf{Output:} updated parameters $\theta$
\STATE $r \leftarrow \textsc{dist.get\_rank}()$
\STATE \textsc{dist.broadcast}$(\theta,\text{src}=0)$
\IF{$r=0$}
  \STATE Sample seeds $S=\{s_i\}_{i=1}^{n_{pert}}$ and scatter across ranks
\ENDIF
\STATE \textsc{dist.scatter}$(S,\text{src}=0)$

\STATE \textbf{Compute center loss once:} $L_0 \leftarrow \sum_{t=1}^{a} \mathcal{L}(\theta,x_t)$

\FOR{$i=1$ to $n$}
  \IF{$r=i$}
  \STATE $z_i \leftarrow \textsc{Packedperturbation}(s_i)$
  \STATE $\theta \leftarrow \textsc{ApplyPackedperturbation}(\theta,\;+\epsilon,\;z_i)$
  \STATE $L_{+} \leftarrow \sum_{t=1}^{a} \mathcal{L}(\theta,x_t)$
  \STATE $\theta \leftarrow \textsc{ApplyPackedperturbation}(\theta,\;-2\epsilon,\;z_i)$
  \STATE $L_{-} \leftarrow \sum_{t=1}^{a} \mathcal{L}(\theta,x_t)$
  \STATE $\theta \leftarrow \textsc{ApplyPackedperturbation}(\theta,\;+\epsilon,\;z_i)$
  \STATE \textsc{dist.gather}$([L_{+},L_{-}],\text{dst}=0)$
  \ENDIF
\ENDFOR

\IF{$r=0$}
  \FOR{$i=1$ to $n$}
    \STATE $z_i \leftarrow \textsc{Packedperturbation}(s_i)$
    \STATE $\hat{c}_i \leftarrow (L_{+,i} - 2L_0 + L_{-,i})/\epsilon^2$
    \STATE $g_i \leftarrow (L_{+,i} - L_{-,i})$
    \STATE $w_i \leftarrow 1/\max(\lambda_{\mathrm{reg}},|\hat{c}_i|^\alpha)$
    \STATE $\gamma_i \leftarrow \dfrac{1}{2n}\, g_i\, w_i$
    \STATE $\theta \leftarrow \textsc{ApplyPackedperturbation}(\theta,\;-\gamma_i,\;z_i)$
  \ENDFOR
\ENDIF
\end{algorithmic}
\end{algorithm}

This results in training that uses nearly the same memory as inference and wall-clock per step that can be scaled up or down depending on the availability of compute. We outline our full implementation in \cref{alg:dist_15spsa}. In practice, we train on a single 8xA100 node and perturbation forward-passes are fully parallelized across 8 ranks, reducing the sequential forward-pass computations by 8$\times$ vs. full parallelization. Rank 0 sends seeds to all other ranks, and each rank applies its assigned perturbations locally, computes losses, and communicates only scalar loss values back to rank 0 to update $\theta$. After the update, the updated $\theta$ must be communicated to all ranks but this happens only tens of times with large batch size and $n_{pert}$. Additionally, its important to note that total cluster memory use actually scales by number of ranks (r) and micro batch size (m) and model size $|\theta|$ providing: $O(r \space m \space|\theta|)$ memory use, as is typical in Distributed Data Parallelism (DDP).  

\section{Experiments}
\label{sec:experiments}
This section reports results for all experiments conducted. 

\subsection{OPT-13B/30B Post-Training MeZO comparison}
First, we observe 1.5-SPSA for post-training LLMs on OPT-13B for direct comparison to MeZO. We also scale this up to OPT-30B to see how we compare to MeZO at larger model sizes. We emphasize that our BP baselines are drawn from standard post-training practice and are not exhaustively re-optimized for the extreme large-batch, few-step regime explored in this section. Our goal is not to claim universal superiority over backpropagation, but rather to demonstrate that 1.5-SPSA outperforms previous zero-order methods, and can achieve competitive or superior performance to BP in specific compute allocation as is done in the MeZO configuration.

\begin{table}[H]
\centering
\scriptsize
\setlength{\tabcolsep}{2pt}
\renewcommand{\arraystretch}{1.05}
\begin{tabular}{@{}lcccccc@{}}
\toprule
\textbf{Method} & \textbf{SST-2} & \textbf{RTE} & \textbf{BoolQ} & \textbf{WSC} & \textbf{WiC}   \\
\midrule
Zero-shot & 58.8 & 59.6 & 59.0 & 38.5 & 55.0  \\
ICL & 87.0 & 62.1 & 66.9 & 39.4 & 50.5   \\
LP & 93.4 & 68.6 & 59.3 & 63.5 & 60.2  \\
MeZO~\citep{malladi2022mezo} & 91.4 & 66.1 & 67.6 & 63.5 & 61.1  \\
MeZO (LoRA) & 89.6 & 67.9 & 73.8 & 64.4 & 59.7   \\
MeZO (prefix) & 90.7 & 70.8 & 73.1 & 60.6 & 59.9 \\
BP+Adam & 92.0 & 70.8 & \textbf{77.1} & 63.5 & \textbf{70.1}   \\
\midrule
1SPSA   & 94.2 & 63.3 & 76.5 & 65.4 & 61.8  \\
1.5-SPSA  & \textbf{94.5} & \textbf{77.7} & 76.5 & \textbf{71.2} & 61.9  \\
\bottomrule
\end{tabular}
\caption{
OPT-13B post-training accuracy results on common GLUE/SuperGLUE tasks. Rows above the midrule are baselines. All 1SPSA and 1.5-SPSA were trained with less than 300 optimization steps. 1.5-SPSA has superior performance. 
}
\label{tab:opt13b_tasks}
\end{table}

\begin{table}[H]
\centering
\scriptsize
\setlength{\tabcolsep}{1.5pt}
\renewcommand{\arraystretch}{1.05}
\begin{tabular}{@{}lcccccc@{}}
\toprule
\textbf{Method} & \textbf{SST-2} & \textbf{RTE} & \textbf{BoolQ} & \textbf{WSC} & \textbf{WiC}  \\
\midrule
MeZO / Prefix & 90.6 & 72.6 & 73.5 & 63.5 & 59.1   \\
\midrule
1SPSA & 94.0 & 69.0 & 73.0 & \textbf{67.5} & \textbf{59.9}  \\
\midrule
1.5-SPSA   & \textbf{94.5} & \textbf{77.0} & \textbf{74.0} & \textbf{67.5} & 59.3  \\ 
\bottomrule
\end{tabular}
\caption{
OPT-30B post-training accuracy results on common GLUE/SuperGLUE tasks. All 1SPSA and 1.5-SPSA were trained with less than 300 optimization steps. 1.5-SPSA proves superior performance. 
}
\label{tab:OPT-30B_tasks}
\end{table}

\subsection{Qwen3 Post-Training}

To ensure these results are not just a property of OPT, we train on a more modern architecture. We train Qwen3-8B on the same benchmarks as the OPT trials, Also, we train Qwen3-1.7B \citep{yang2025qwen3} on Stable ToolBench \citep{guo2024stabletoolbench} to see how it performs on a different task vs. those tested in MeZO. 1.5-SPSA outperforms 1SPSA by 0.4\%, and BP with Adam by 2\%. Additionally, we perform a sweep of learning rates $\lambda$s to demonstrate the stability of 1.5-SPSA over 1SPSA and MeZO at much higher learning rates which allows for larger jumps in parameter space. 

\begin{table}[H]
\centering
\scriptsize
\setlength{\tabcolsep}{2pt}
\renewcommand{\arraystretch}{1.05}
\begin{tabular}{@{}lcccccc@{}}
\toprule
\textbf{Method} & \textbf{SST-2} & \textbf{RTE} & \textbf{BoolQ} & \textbf{WSC} & \textbf{WiC}  \\
\midrule
1SPSA       & 94.5 & \textbf{88.5} & 85.7 & 75.0  & 64.6 \\
\midrule
1.5-SPSA   & \textbf{94.7} & 88.0 & \textbf{86.1} & \textbf{80.8} & \textbf{71.2}   \\
\bottomrule
\end{tabular}
\caption{
Qwen3-8B post-training accuracy results on common GLUE/SuperGLUE tasks. All 1SPSA and 1.5-SPSA were trained with less than 300 optimization steps. 1.5-SPSA proves superior performance. 
}
\label{tab:qwen3_8b_tasks}
\end{table}

\begin{table}[H]
\centering
\scriptsize
\setlength{\tabcolsep}{3pt}
\begin{tabular}{@{}lcccccc@{}}
\toprule
\textbf{Method} & \textbf{lr=1e-3} & \textbf{lr=5e-4} & \textbf{lr=1e-4} & \textbf{lr=5e-5} & \textbf{lr=1e-5} & \textbf{lr=5e-6} \\
\midrule
\shortstack{BP+Adam \\ \hfill}         & \shortstack{72.0 \\ (70)} & \shortstack{70.2 \\ (80)} & \shortstack{77.0 \\ (285)} & \shortstack{77.0 \\ (284)} & \shortstack{70.2 \\ (279)} & \shortstack{65.3 \\ (275)} \\
\midrule
\shortstack{1SPSA \\ \hfill}      & \shortstack{$diverge$ \\ \hfill }     & \shortstack{$diverge$ \\ \hfill }      & \shortstack{78.6 \\ (38)}     & \shortstack{78.5 \\ (38)}     & \shortstack{77.2 \\ (44)}     & \shortstack{77.2 \\ (45)} \\
\midrule
\shortstack{1.5-SPSA \\ \hfill}  & \shortstack{77.0 \\ (24)}    & \shortstack{ \textbf{79.0} \\ \textbf{(24)} }    & \shortstack{78.0 \\ (37)}     & \shortstack{77.6 \\ (37)}     & \shortstack{77.1 \\ (39)}     & \shortstack{77.0 \\ (44)} \\
\bottomrule
\end{tabular}
\caption{Qwen3-1.7B post-training accuracy results on Stable ToolBench to teach tool-use. Format: Test Acc (on top), optimization steps to converge (on bottom). All 1SPSA and 1.5-SPSA were trained with less than 300 optimization steps. We show 1.5-SPSA outperforms 1SPSA and BP with Adam on Qwen3. Additionally, we show that 1.5-SPSA permits larger learning rates which improves convergence rate.}
\label{tab:qwen_lrs}
\end{table}

\subsection{Convex Optimization: Toy Stiff Paraboloid}
We hypothesize that the benefit of curvature preconditioning scales with the condition number of the local Hessian. To test this, we construct a synthetic stiff paraboloid objective $J(x) = x^\top R^\top \text{diag}(k, 1, \dots, 1) R x$, where $R$ is a random orthonormal matrix and $k$ controls the condition number $\kappa=\frac{k}{1}$. We vary $k \in \{1, 5, 10, 50, 100, 500, 1000\}$ and measure the number of optimization steps required to converge to the global minimum $x^*$. As illustrated in \cref{fig:paraboloid_runs}, 1.5-SPSA matches 1SPSA's convergence rate at low $k$ (well-conditioned) but converges significantly faster as $k$ increases (ill-conditioned), confirming our hypothesis.

\begin{figure}[!t]
    \centering
    \includegraphics[width=0.48\textwidth]{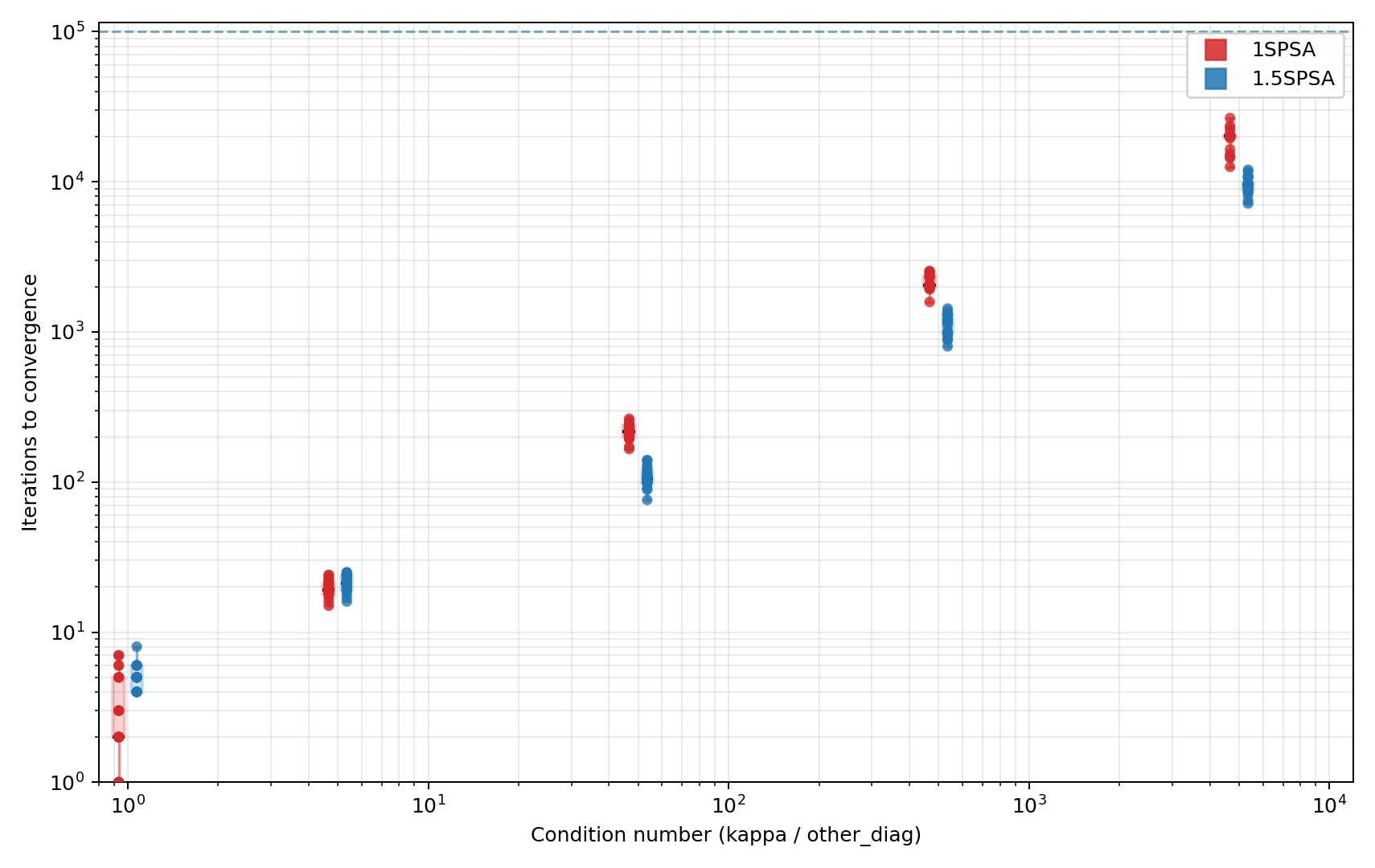}
    \caption{Comparing convergence rates on a stiff paraboloid $x^\top R^\top \text{diag}(k, 1, \dots, 1) R x$ at varying condition numbers ($\kappa = k/1$). 10 seeds per $\kappa \times solver$. Random R and perturbations to compare 1SPSA vs. 1.5-SPSA convergence rates. As $\kappa$ grows, 1.5-SPSA converges faster (up to 7$\times$ on average).}
    \label{fig:paraboloid_runs}
\end{figure}

\subsection{Non-Convex Optimization: DNC overfitting stress tests across scale}
Differentiable Neural Computers (DNCs), as introduced by \citep{graves2016hybrid}, are a class of Recurrent Neural Networks (RNNs) that are notoriously difficult to train as they include an external memory that can be read from and written to at each step as well as a hidden memory. We check to see if 1.5-SPSA can outperform 1SPSA and BPTT in convergence rate on a fixed batch until loss reaches a near-zero threshold. We report convergence rate (steps-to-0) across model scales up to 1B parameters, and also show the impact of increasing $n_{pert}$. As shown in ~\cref{fig:overfit_15spsa_vs_1spsa}, while overfitting a DNC model, the larger we make $n_{pert}$, and as we increase the model size, the faster convergence rate. We observe that 1.5-SPSA converges much faster than 1SPSA (by as much as 6$\times$ less steps) and in many cases faster the BPTT at larger perturbations per step.

\begin{figure}[!t]
\centering
\begin{tikzpicture}
\begin{loglogaxis}[
  title={DNC Steps-to-Zero Loss},
  title style={font=\small, yshift=2pt},
  xlabel={Model size (parameters)},
  ylabel={Steps to zero loss},
  xmajorgrids=true,
  ymajorgrids=true,
  grid style=dashed,
  width=8.3cm,
  height=6.0cm,
  xtick={3e5,1e6,4e6,1.7e7,6.7e7,2.7e8,1.1e9},
  xticklabels={300k,1M,4M,17M,68M,270M,1.1B},
  legend style={
    at={(0.5,-0.28)},
    anchor=north,
    draw=none,
    font=\scriptsize,
    row sep=-2pt,
    column sep=4pt,
    cells={anchor=west},
    inner xsep=1pt,
    inner ysep=1pt,
  },
  legend columns=3,
  xlabel style={font=\small},
  ylabel style={font=\small},
]

\addplot+[mark=*, thick, color=black] table {
304357 2870
1132901 1180
4362853 480
17114213 200
67782757 90
269783141 37
};
\addlegendentry{BPTT}

\addplot+[mark=triangle*, thick, color=blue, solid] table {
304357 4635
1132901 2514
4362853 1386
17114213 660
67782757 525
269783141 672
1076437093 809
};
\addlegendentry{1SPSA @ 8}

\addplot+[mark=square*, thick, color=red, solid] table {
304357 1180
1132901 600
4362853 250
17114213 180
67782757 120
269783141 80
1076437093 110
};
\addlegendentry{1SPSA @ 96}

\addplot+[mark=diamond*, thick, color=green!60!black, solid] table {
304357 150
1132901 103
4362853 52
17114213 22
67782757 16
269783141 20
1076437093 24
};
\addlegendentry{1SPSA @ 512}

\addplot+[mark=triangle*, very thick, color=blue, densely dotted, mark options={solid}] table {
304357 344
1132901 319
4362853 332
17114213 363
67782757 281
269783141 294
1076437093 394
};
\addlegendentry{1.5-SPSA @ 8}

\addplot+[mark=square*, very thick, color=red, densely dotted, mark options={solid}] table {
304357 94
1132901 51
4362853 47
17114213 34
67782757 27
269783141 39
1076437093 40
};
\addlegendentry{1.5-SPSA @ 96}

\addplot+[mark=diamond*, very thick, color=green!60!black, densely dotted, mark options={solid}] table {
304357 80
1132901 60
4362853 42
17114213 22
67782757 16
269783141 11
1076437093 9
};
\addlegendentry{1.5-SPSA @ 512}

\end{loglogaxis}
\end{tikzpicture}
\caption{
DNC overfitting stress test: steps to reach a near-zero loss for 7 different model sizes up to 1B. 1.5-SPSA reduces steps over 1SPSA at the same perturbation count. More perturbation count improves convergence rate. Larger models converge faster. BPTT cannot run the 1.1B model in this setup due to GPU memory limits.
}
\label{fig:overfit_15spsa_vs_1spsa}
\end{figure}
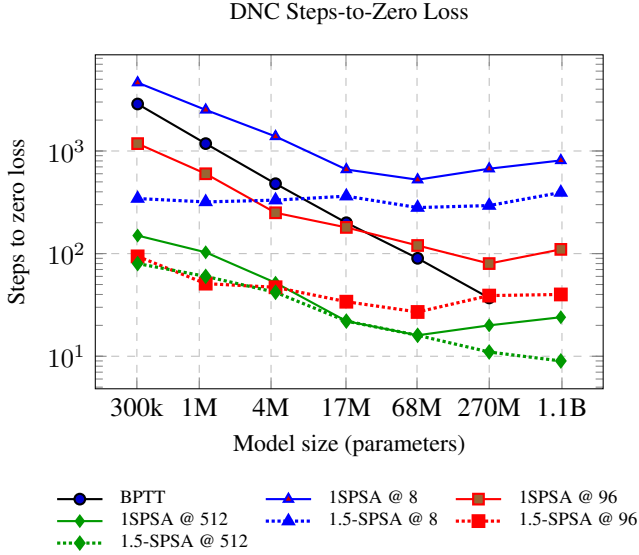

\section{Discussion and Future Work}

\paragraph{Limitations and scope of comparison to backpropagation.}
A limitation of this study is that we do not exhaustively retune backpropagation (BP) baselines under the same extreme compute-allocation regimes explored for 1SPSA and 1.5-SPSA. While such tuning is possible in principle, it typically requires substantial optimizer state, gradient checkpointing, or backward-pass memory, which conflicts with our desiderata. For example, we may achieve considerable reduction in memory using gradient checkpointing at the cost of additional compute. We therefore interpret our results not as a claim that 1.5-SPSA universally outperforms BP on memory, compute or accuracy, but as evidence that when compute is allocated normally, 1.5-SPSA is highly competitive.

\paragraph{Choice of batch size and $n_{pert}$}

We perform a detailed analysis of curvature, batch noise and perturbation noise in \cref{sec:empirical_analysis} to help us understand the relationship between our two most important hyperparameters and gradient noise. We can combat batch noise and perturbation noise by increasing batch size and $n_{pert}$, trading off compute for stability, but how much is enough? We find that stability is related to the median $\hat{g}(\theta)$. We analyze each separately. First, in \cref{fig:batch_noise_var}, which shows when batch noise reduces below the median gradient, we see that convergence does well once our signal to noise ratio approaches and surpasses 1. This is the point when we start to see stable convergence in \cref{fig:spsa_stacked_heatmaps}. Similarly, in \cref{fig:perturbation_variance}, which shows when perturbation noise reduces below the median gradient, we see that convergence does well once our signal to noise ratio approaches and surpasses 1. Rather than brute-force search through these hyperparameters, we can use this technique to find the minimally sufficient batch size and $n_{pert}$ that will surpass SNR=1. 

\paragraph{Choice of $\lambda$.}
The learning rate $\lambda$ plays a critical role in realizing the benefits of 1SPSA and 1.5-SPSA. If $\lambda$ is too small, curvature scaling becomes ineffective and instability can reappear. While we brute force search for $\lambda$, a more systematic approach (linear, binary, quadratic fit search) may significantly reduce tuning overhead and improve robustness, as well as offer insight into a better learning rate schedule than plateau decay.

\paragraph{Choice of $\alpha$.}
We find that $\alpha = 0.1$ is sufficiently optimal throughout tasks and architectures, which yields mild curvature weighting as shown in Appendix~\Cref{app:alpha_sweep}. During training, as shown in \cref{fig:spaghetti}, we observe 3-point curvatures as large as $10^8$. Note, $|\hat{c}|^{\alpha=0.1} \approx 6.3$. So under our weighting scheme, we down-weight this direction 6$\times$ more than zero-curvature direction. This mild down-weighting is sufficient to remove instability. Furthermore, these stability gains compound across iterations. 

\paragraph{Batch-normalized curvature normalization.}
Our current curvature scaling operates on per-perturbation curvature estimates without normalization across the batches and/or perturbations. As shown in \cref{fig:qwen3_curvature_histogram}, it may be appropriate to normalize the curvature per batch, or per perturbation, or per step. Exploring batch-normalized or relative curvature estimates is a natural extension and may further stabilize updates without introducing additional hyper-parameters or memory use. We leave this to future work.

\section{Conclusion}

We presented 1.5-SPSA, a memory-efficient zero-order optimizer that bridges the gap between first-order and zero-order training efficacy. By maintaining constant training compute, we showed that 1SPSA and 1.5-SPSA perform better with larger effective batch sizes and many perturbations, not many steps. To tame the resulting curvature variance, algorithm 1.5-SPSA adds a single clean forward-pass to precondition updates in perturbation space. This yields State-of-the-Art results on OPT post-training for ZOO with orders of magnitude less compute compared to prior ZOO methods.

\clearpage
\bibliographystyle{icml2026}
\bibliography{references}  %

\begin{thebibliography}{19}
\providecommand{\natexlab}[1]{#1}
\providecommand{\url}[1]{\texttt{#1}}
\expandafter\ifx\csname urlstyle\endcsname\relax
  \providecommand{\doi}[1]{doi: #1}\else
  \providecommand{\doi}{doi: \begingroup \urlstyle{rm}\Url}\fi

\bibitem[Chen et~al.(2024)Chen, Zhang, Jia, Diffenderfer, Liu, Parasyris,
  Zhang, Zhang, Kailkhura, and Liu]{chen2024deepzero}
Chen, A., Zhang, Y., Jia, J., Diffenderfer, J., Liu, J., Parasyris, K., Zhang,
  Y., Zhang, Z., Kailkhura, B., and Liu, S.
\newblock Deepzero: Scaling up zeroth-order optimization for deep model
  training.
\newblock In \emph{International Conference on Learning Representations
  (ICLR)}, 2024.

\bibitem[Chen et~al.(2019)Chen, Liu, Xu, Li, Lin, Hong, and
  Cox]{chen2019zo_adamm}
Chen, X., Liu, S., Xu, K., Li, X., Lin, X., Hong, M., and Cox, D.
\newblock Zo-adamm: Zeroth-order adaptive momentum method for black-box
  optimization.
\newblock In \emph{Advances in Neural Information Processing Systems}, 2019.

\bibitem[Choromanski et~al.(2019)Choromanski, Pacchiano, Parker-Holder, Tang,
  and Sindhwani]{cottrell2020asebo}
Choromanski, K., Pacchiano, A., Parker-Holder, J., Tang, Y., and Sindhwani, V.
\newblock From complexity to simplicity: Adaptive es-active subspaces for
  blackbox optimization.
\newblock \emph{arXiv preprint arXiv:1912.01255}, 2019.

\bibitem[Dauphin et~al.(2014)Dauphin, Pascanu, Gulcehre, Cho, Ganguli, and
  Bengio]{dauphin2014saddle}
Dauphin, Y.~N., Pascanu, R., Gulcehre, C., Cho, K., Ganguli, S., and Bengio, Y.
\newblock Identifying and attacking the saddle point problem in
  high-dimensional non-convex optimization.
\newblock In \emph{Advances in Neural Information Processing Systems}, 2014.

\bibitem[Gautam et~al.(2024)Gautam, Park, Zhou, Raman, and
  Ha]{gautam2024variancereducedzerothordermethodsfinetuning}
Gautam, T., Park, Y., Zhou, H., Raman, P., and Ha, W.
\newblock Variance-reduced zeroth-order methods for fine-tuning language
  models, 2024.
\newblock URL \url{https://arxiv.org/abs/2404.08080}.

\bibitem[Ghorbani et~al.(2019)Ghorbani, Krishnan, and
  Xiao]{ghorbani2019investigation}
Ghorbani, B., Krishnan, S., and Xiao, Y.
\newblock An investigation into neural net optimization via hessian eigenvalue
  density.
\newblock \emph{arXiv preprint arXiv:1901.10159}, 2019.

\bibitem[Graves et~al.(2016)Graves, Wayne, Reynolds, Harley, Danihelka,
  Grabska-Barwinska, Gomez, Grefenstette, Ramalho, Agapiou,
  et~al.]{graves2016hybrid}
Graves, A., Wayne, G., Reynolds, M., Harley, T., Danihelka, I.,
  Grabska-Barwinska, A., Gomez, S., Grefenstette, E., Ramalho, T., Agapiou, J.,
  et~al.
\newblock Hybrid computing using a neural network with dynamic external memory.
\newblock \emph{Nature}, 538\penalty0 (7626):\penalty0 471--476, 2016.
\newblock \doi{10.1038/nature20101}.
\newblock URL \url{https://doi.org/10.1038/nature20101}.

\bibitem[Guo et~al.(2024)Guo, Cheng, Wang, Liang, Qin, Li, Liu, Sun, and
  Liu]{guo2024stabletoolbench}
Guo, Z., Cheng, S., Wang, H., Liang, S., Qin, Y., Li, P., Liu, Z., Sun, M., and
  Liu, Y.
\newblock Stabletoolbench: Towards stable large-scale benchmarking on tool
  learning of large language models, 2024.
\newblock URL \url{https://arxiv.org/abs/2403.07714}.

\bibitem[Johnson \& Lindenstrauss(1984)Johnson and
  Lindenstrauss]{JohnsonLindenstrauss1984}
Johnson, W.~B. and Lindenstrauss, J.
\newblock Extensions of {Lipschitz} mappings into a {Hilbert} space.
\newblock In Beals, R., Beck, A., Bellow, A., and Hajian, A. (eds.),
  \emph{Conference in Modern Analysis and Probability}, volume~26 of
  \emph{Contemporary Mathematics}, pp.\  189--206. American Mathematical
  Society, Providence, RI, 1984.
\newblock \doi{10.1090/conm/026/737400}.

\bibitem[Kingma \& Ba(2015)Kingma and
  Ba]{kingma2017adammethodstochasticoptimization}
Kingma, D.~P. and Ba, J.
\newblock Adam: A method for stochastic optimization.
\newblock In \emph{International Conference on Learning Representations
  (ICLR)}, 2015.

\bibitem[Malladi et~al.(2023)Malladi, Gao, Nichani, Damian, Lee, Chen, and
  Arora]{malladi2022mezo}
Malladi, S., Gao, T., Nichani, E., Damian, A., Lee, J.~D., Chen, D., and Arora,
  S.
\newblock Fine-tuning language models with just forward passes.
\newblock In \emph{Advances in Neural Information Processing Systems}, 2023.
\newblock URL \url{https://arxiv.org/abs/2305.17333}.
\newblock Corrected from early arXiv ID to NeurIPS 2023 version.

\bibitem[Rubinstein \& Kroese(2004)Rubinstein and
  Kroese]{RubinsteinKroese2004CrossEntropy}
Rubinstein, R.~Y. and Kroese, D.~P.
\newblock \emph{The {Cross-Entropy} Method: A Unified Approach to Combinatorial
  Optimization, {Monte-Carlo} Simulation and Machine Learning}.
\newblock Information Science and Statistics. Springer, New York, NY, 2004.
\newblock \doi{10.1007/978-1-4757-4321-0}.

\bibitem[Sagun et~al.(2017)Sagun, Evci, Guney, Dauphin, and
  Bottou]{sagun2017empirical}
Sagun, L., Evci, U., Guney, V.~U., Dauphin, Y., and Bottou, L.
\newblock Empirical analysis of the hessian of over-parametrized neural
  networks.
\newblock \emph{arXiv preprint arXiv:1706.04454}, 2017.

\bibitem[Salimans et~al.(2017)]{salimans2017evolution}
Salimans, T. et~al.
\newblock Evolution strategies as a scalable alternative to reinforcement
  learning.
\newblock \emph{arXiv}, 2017.

\bibitem[Spall(1992)]{spall1992multivariate}
Spall, J.~C.
\newblock Multivariate stochastic approximation using a simultaneous
  perturbation gradient approximation.
\newblock \emph{IEEE Transactions on Automatic Control}, 37\penalty0
  (3):\penalty0 332--341, 1992.
\newblock \doi{10.1109/9.119632}.

\bibitem[Spall(1997)]{spall1997accelerated}
Spall, J.~C.
\newblock Accelerated second-order stochastic optimization using only function
  measurements.
\newblock In \emph{Proceedings of the 36th IEEE Conference on Decision and
  Control}, volume~2, pp.\  1417--1424. IEEE, 1997.
\newblock \doi{10.1109/CDC.1997.657661}.

\bibitem[Wang et~al.(2024)Wang, Shen, Ding, Xue, Liu, and
  Ding]{wang2024simultaneous}
Wang, F., Shen, L., Ding, L., Xue, C., Liu, Y., and Ding, C.
\newblock Simultaneous computation and memory efficient zeroth-order optimizer
  for fine-tuning large language models.
\newblock \emph{arXiv preprint arXiv:2410.09823}, 2024.

\bibitem[Wierstra et~al.(2008)]{wierstra2008nes}
Wierstra, D. et~al.
\newblock Natural evolution strategies.
\newblock \emph{Journal of Machine Learning Research}, 2008.

\bibitem[Yang et~al.(2025)]{yang2025qwen3}
Yang, A. et~al.
\newblock Qwen3: The next generation of unified large language models.
\newblock \emph{arXiv preprint arXiv:2505.09388}, 2025.

\end{thebibliography}

\newpage
\appendix

\section{Appendix}
\subsection{Empirical Analysis on Qwen3-8B Loss Landscape}
In this section, we analyze the Qwen3-8B Loss Landscape subject to SST-2. 

\label{sec:empirical_analysis}

\subsubsection{1-D Loss Landscape}

We first attempt to understand the loss landscape. We take random directions $z_i$ and plot 40 points along the that direction. We do this 100 times and show the line plots in \cref{fig:spaghetti}. As you can, the lines are very non-convex, with high variance in curvature. 

\begin{figure}[H]
    \centering
    \includegraphics[width=0.48\textwidth]{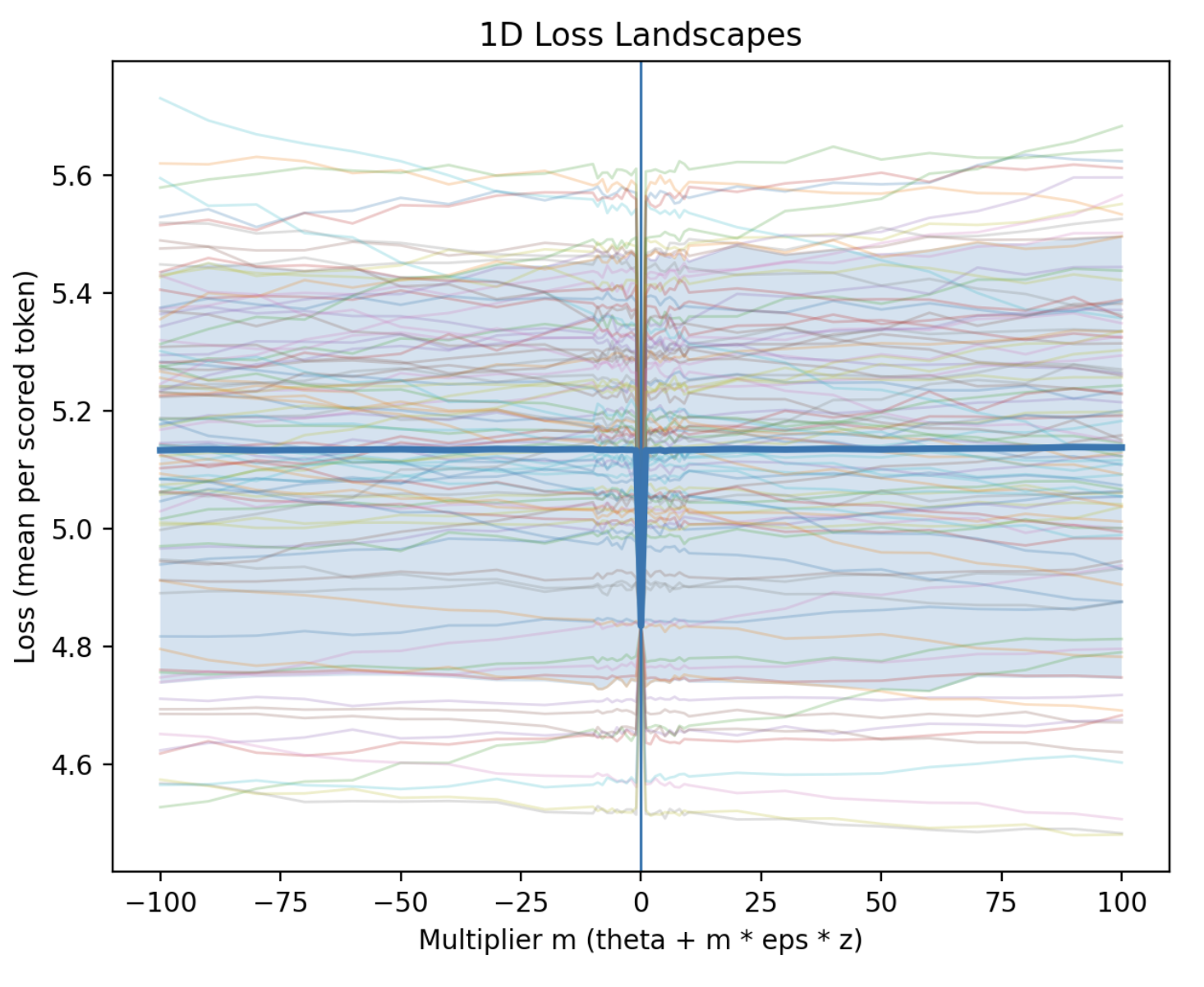}
    \caption{We perturb Qwen3-8B in random directions and plot the 1-D loss landscape at 40-points (SST-2, batch size = 256). The resulting ``spaghetti plot" shows an \textbf{indefinite and ill-conditioned Hessian}: local curvature ranges from sharply negative to sharply positive (up to $10^8$). }
    \label{fig:spaghetti}
\end{figure}

\subsubsection{3-point Curvature Histogram}
Next, we attempt to characterize the distribution of our 3-point curvature. We probe eight thousand random directions and calculate $\hat{c}$ as defined above and plot the histogram in \cref{fig:qwen3_curvature_histogram}. As shown, the values can range by eight orders of magnitude. This demonstrates that our loss function is very ill-conditioned and suggests preconditioning would be useful for faster convergence. 

\begin{figure}[H]
    \centering
    \includegraphics[width=0.48\textwidth]{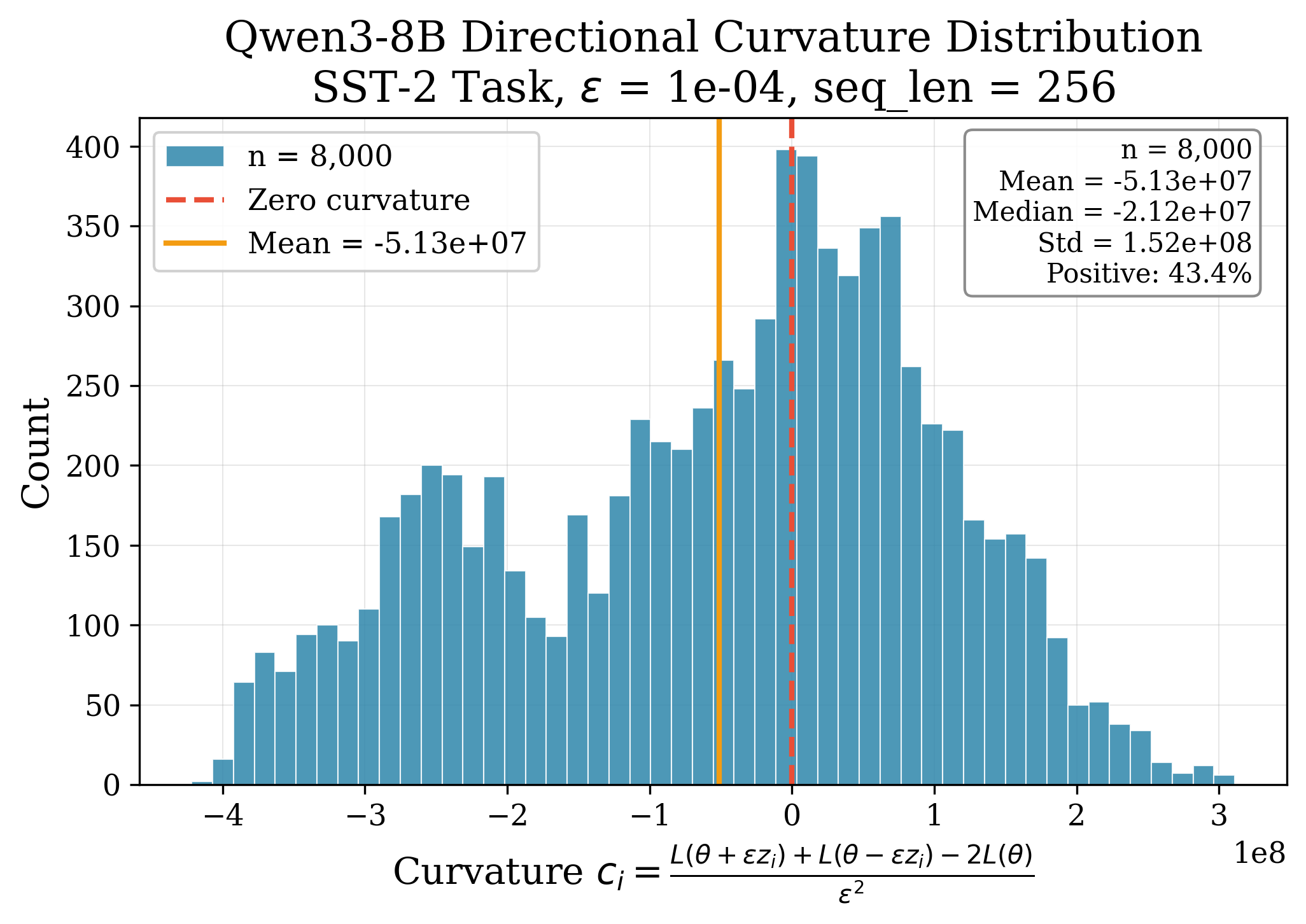}
    \caption{ \textbf{Histogram of 3-point curvature of Qwen3-8B loss landscape against SST-2.} We see values ranging from $-40^8$ to $30^8$ proving the post-training loss is ill-conditioned. We use batch size = 128, sequence length = 256, and $\epsilon=10^{-4}$ for eight thousand random perturbation directions. }
    \label{fig:qwen3_curvature_histogram}
\end{figure}

\subsubsection{Batch Noise Analysis}
Next, we want to understand the impact of batch size on our gradient estimate. Intuitively, we know we can decrease the noise in our gradient estimate with larger batch size, but as a practitioner, we still need to know how big is big enough for us to achieve stable convergence. We plot the distribution of our gradient estimate at different batch sizes in \cref{fig:batch_noise_boxplot}. Additionally, in \cref{fig:batch_noise_var}, we also plot the median gradient as well for us to get a sense of the signal we are trying to discern from the noise. The variance in our gradient estimate meets our median gradient around batch size 128 to 256 providing our target range for optimization. We also find that batch size less than this amount leads to unstable training.   
\begin{figure}[H]
    \centering
    \includegraphics[width=0.48\textwidth]{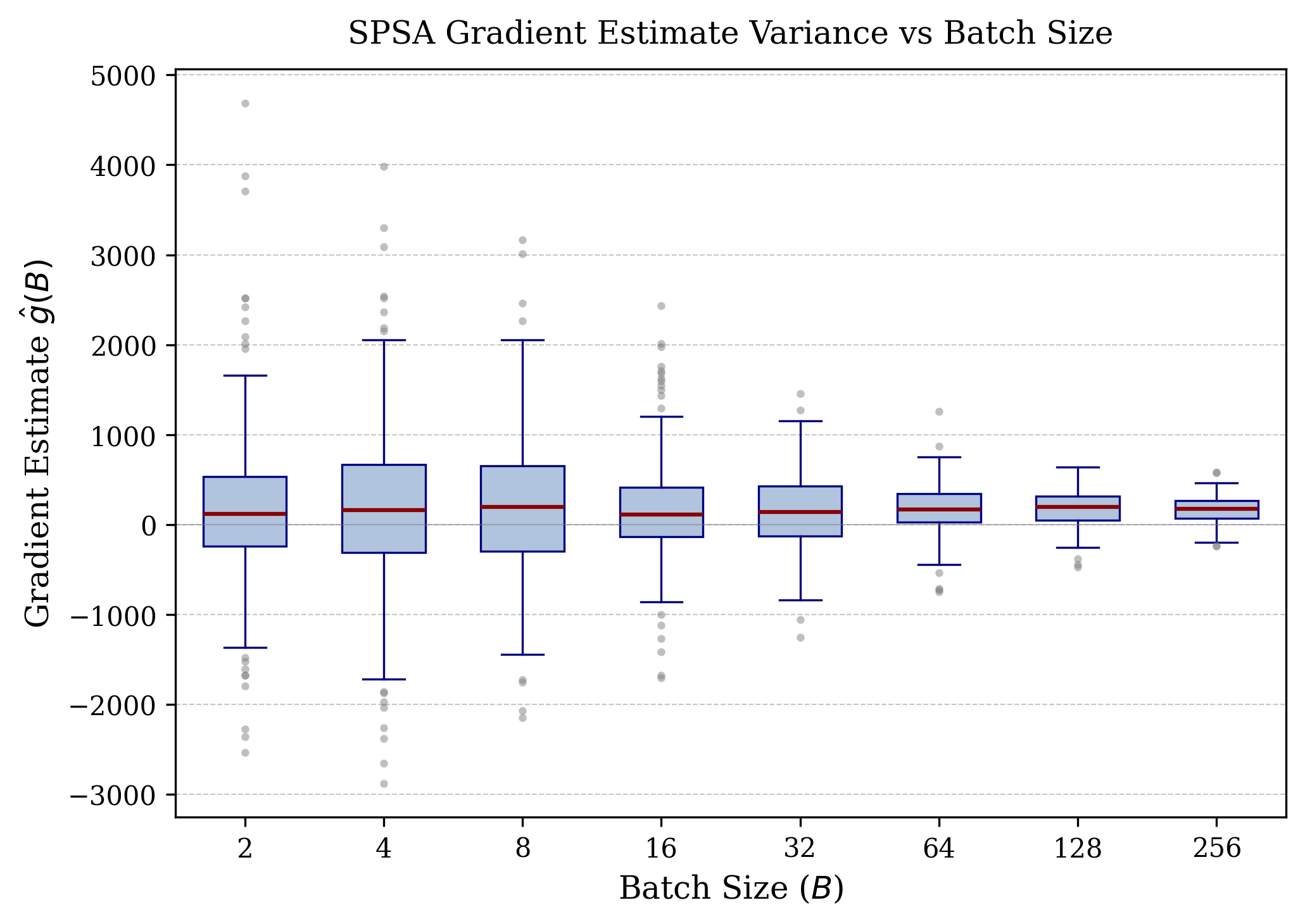}
      \caption{                                     
  \textbf{Box plot showing the distribution of gradient estimates for different batch size. } 
  We measure the finite difference gradient estimate using a Qwen3-8B model on the SST-2 sentiment classification task.                  
  A \emph{single} Rademacher perturbation vector $z \in \{-1, +1\}^d$ is sampled once (seed${}=0$) and held fixed across all measurements.                                                         
  For each batch size $B \in \{2, 4, 8, 16, 32, 64, 128, 256\}$, we sample $n=200$ independent batches $B_j$ from the training set and compute $\hat{g}(B_j)$ using $\epsilon = 10^{-4}$.          
  The same perturbation $z$ and perturbed models $\theta \pm \epsilon z$ are used for every batch; only the data batch varies.  
  } 
    \label{fig:batch_noise_boxplot}
\end{figure}

\begin{figure}[H]
    \centering
    \includegraphics[width=0.48\textwidth]{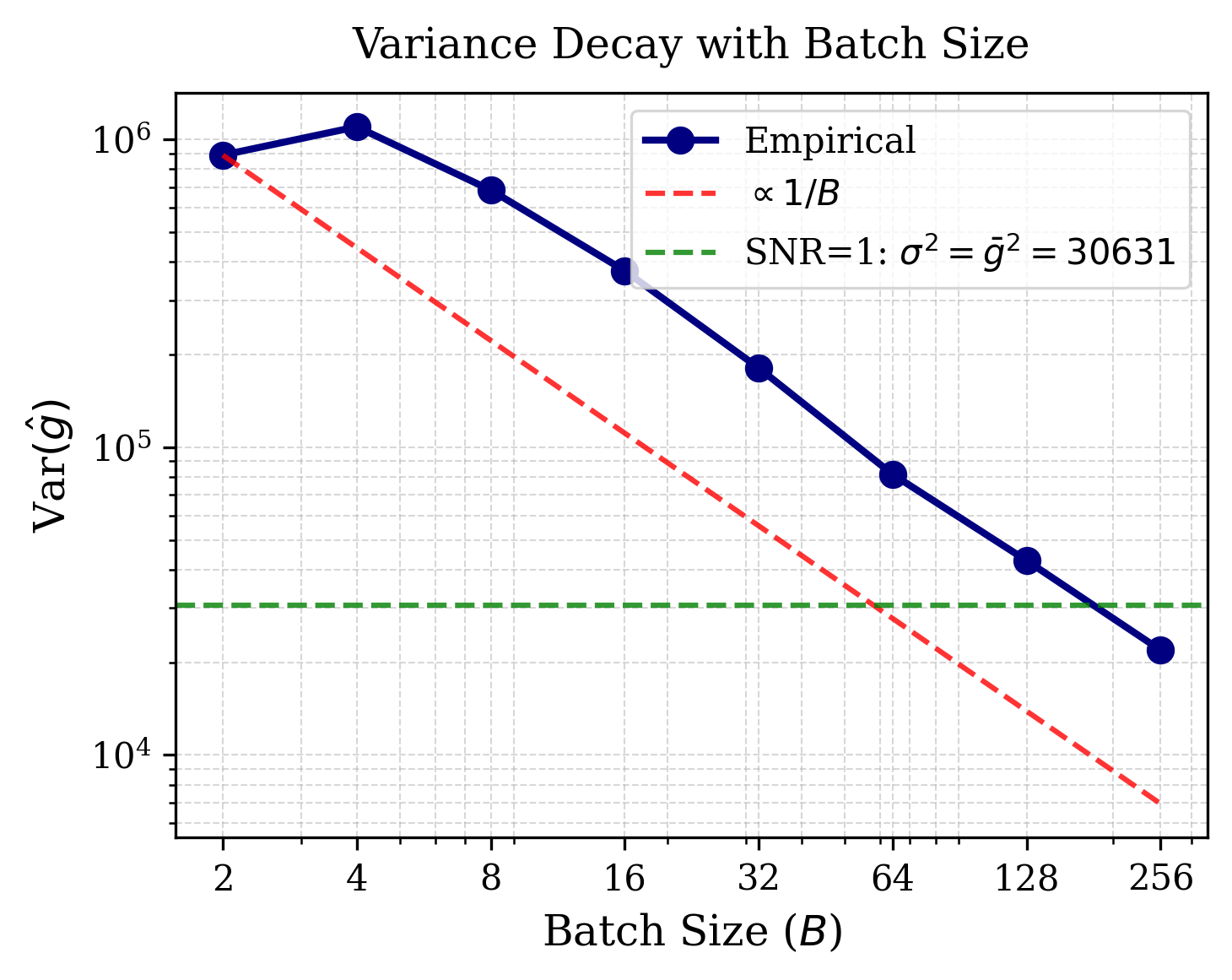}
    \caption{Empirical variance follows the expected $O(1/B)$ decay (red dashed line), consistent with the central limit theorem applied to mini-batch loss averaging.             
  The green dashed line indicates the SNR${}=1$ threshold where the gradient signal magnitude equals the noise standard deviation. }
    \label{fig:batch_noise_var}
\end{figure}

\subsubsection{Perturbation Noise Analysis}

Next, we want to understand the impact of perturbation noise on our gradient estimate the same way we did for batch noise. Intuitively, we know we can decrease the noise in our gradient estimate with larger perturbations, but as a practitioner, we still need to know how big is big enough for us to achieve stable convergence. We plot the distribution of our gradient estimate at different perturbations per step in \cref{fig:perturbation_boxplot}. Additionally, in \cref{fig:perturbation_variance}, we also plot the median gradient as well for us to get a sense of the signal we are trying to discern from the noise. The variance in our gradient estimate meets our median gradient around perturbations 256. 

\begin{figure}[H]
    \centering
    \includegraphics[width=0.48\textwidth]{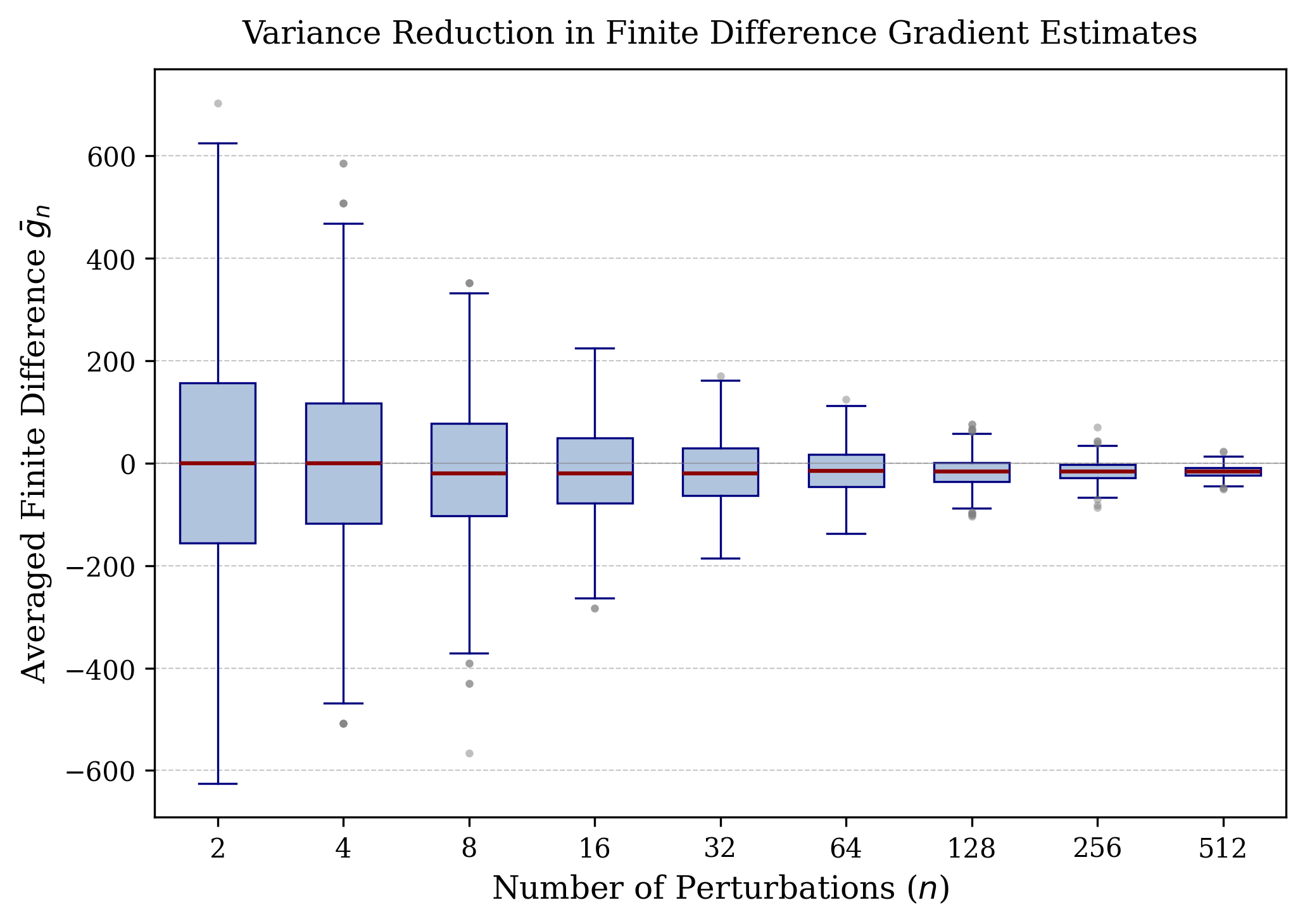}
\caption{ \textbf{Variance reduction in finite difference gradient estimates as a function of the number of perturbations.} We load Qwen3-8B and fix a single batch $\mathcal{B}$ of size 128 from SST-2. We generate 1000 independent Rademacher perturbations $\{z_i\}_{i=1}^{1000}$ and compute our gradient estimate with $\epsilon = 10^{-4}$. The same batch $\mathcal{B}$ is used for all 2000 loss evaluations, isolating perturbation variance from batch variance. For each $n_{pert} \in \{2, 4, 8, 16, 32, 64, 128, 256, 512\}$, we bootstrap 1000 times by sampling $n$ perturbations without replacement. Variance decreases as $O(1/n_{pert})$. } 
    \label{fig:perturbation_boxplot}
\end{figure}

\begin{figure}[H]
    \centering
    \includegraphics[width=0.48\textwidth]{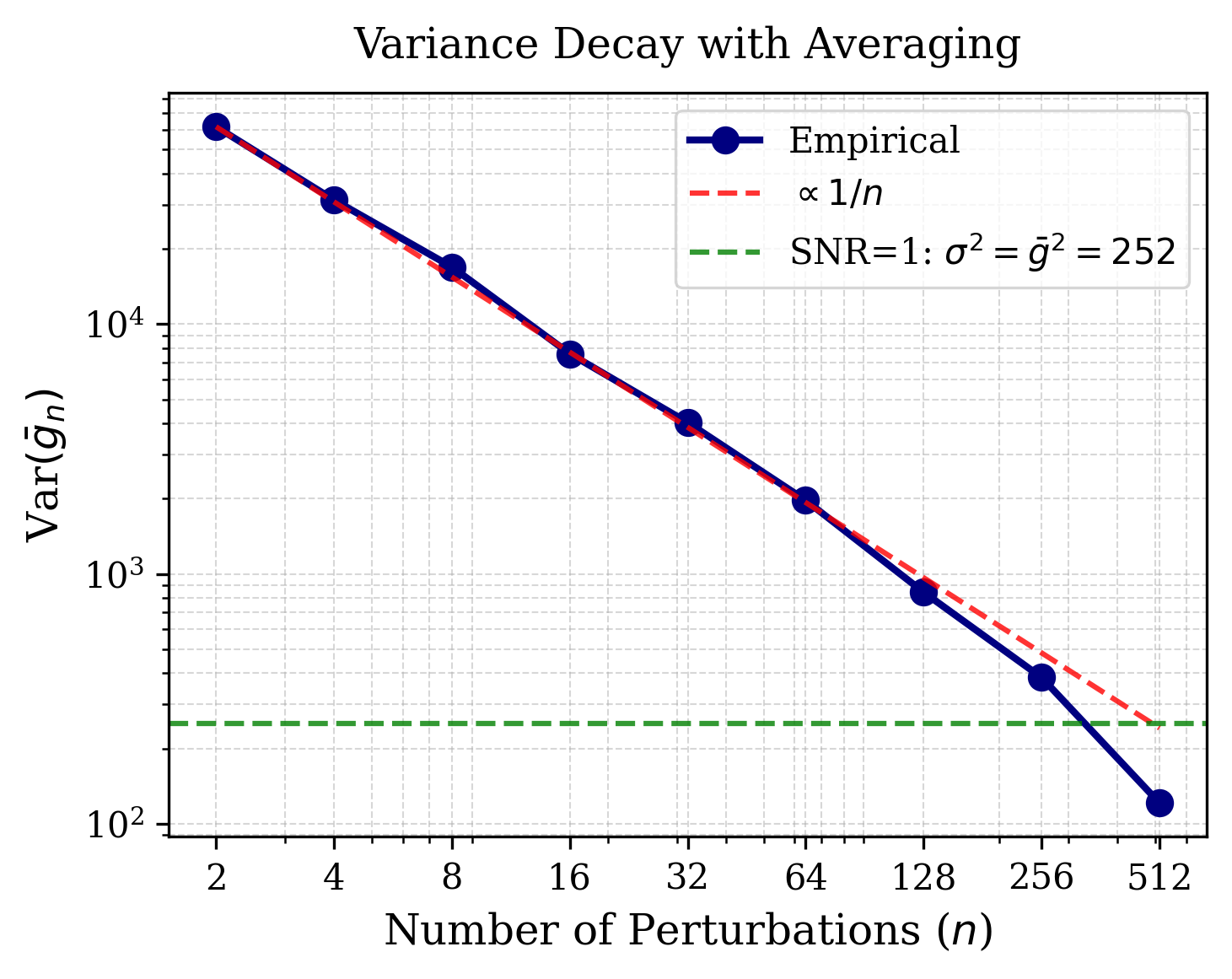}
    \caption{ \textbf{Empirical variance of the averaged finite difference estimator $\bar{g}_n$ as a function of $n_{pert}$.} Same experimental setup as Figure~\ref{fig:perturbation_boxplot}. Blue circles show empirical variance from 1000 bootstrap samples at each $n_{pert}$. Red dashed line shows the theoretical $O(1/n_{pert})$ decay. Green dashed line marks $\text{Var}(\bar{g}_{n_{pert}}) = \bar{g}^2 \approx 252$, the SNR${}=1$ threshold where standard deviation equals signal magnitude. The estimator crosses into the SNR${}>1$ regime around $n_{pert} \approx 256$. Deviation from $1/n_{pert}$ at large $n_{pert}$ is due to finite population effects (sampling 512 of 1000 total perturbations without replacement). }
    \label{fig:perturbation_variance}
\end{figure}

\subsection{Convex Optimization: Stiff Paraboloid Experiments}
To better study the conditions by which 1.5-SPSA outperforms 1SPSA, we construct a toy convex optimization problem. We create a paraboloid that must be rotated randomly as we use the Rademacher distribution and we want to examine which solver outperforms in well-conditioned loss landscapes, where the eigenvalues of the hessian are roughly equal, and which solver outperforms in ill-conditioned loss landscapes, where the eigenvalues of the hessian are orders of magnitude different. We see an example setup in \cref{fig:paraboloid} where we train with each solver until convergence. We randomize this experiment, and run one hundred runs each condition number in a sweep of condition numbers to see a clear pattern as shown in \cref{fig:paraboloid_runs}. As we increase the condition number, 1.5-SPSA shows faster and faster convergence over 1SPSA providing some evidence that the performance increase is tied to the condition number of the hessian. 
\begin{figure}[H]
    \centering
    \includegraphics[width=0.48\textwidth]{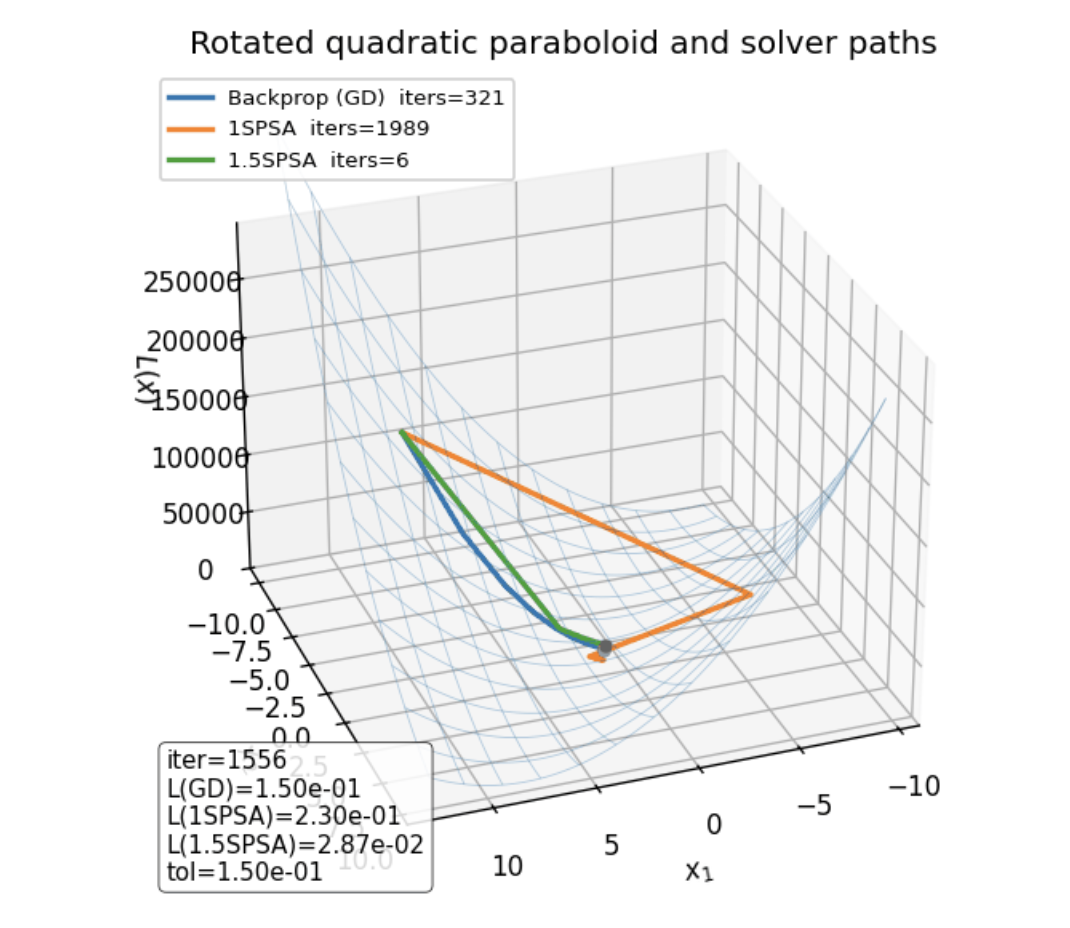}
    \caption{Comparing convergence rates for our "stiff" paraboloid with a $\kappa=100$. 1.5-SPSA converges in just 6 steps while 1SPSA takes $\sim$ 2000 steps. }
    \label{fig:paraboloid}
\end{figure}

\subsection{Non-Convex: DNC Overfit Compute Comparison}

Differentiable Neural Computers (DNCs) are notoriously difficult to train. To test 1.5-SPSA, we train on a sweep of DNC model sizes up to 1B parameters and for each one, overfit on a batch similar to the stiff paraboloid experiment but in this case the loss is massively non-convex. We train to convergence with 1.5-SPSA, 1SPSA, and BPTT (Backpropagation Through Time). We see that 1.5-SPSA outperforms 1SPSA for all runs. We plot in \cref{fig:overfit_15spsa_vs_1spsa_forward_passes} the total forward-passes to achieve near zero train loss, estimating a backward pass equals roughly two forward passes, therefore one step of BPTT costs three forward-passes. Despite this, BPTT mostly outperforms on this task suggesting BPTT is a more compute-efficient solver, at least on this task, yet 1.5-SPSA is more compute-efficient solver than 1SPSA. Additionally, while mostly, BPTT (Backpropagation through time) outperforms, there are cases where 1.5-SPSA requires less compute to achieve near zero loss (e.g. at small model sizes and 512 perturbations per step).

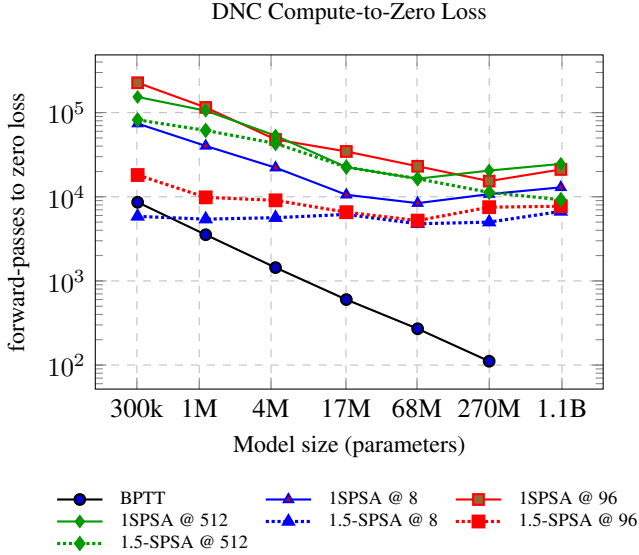
\begin{figure}[H]
\centering
\begin{tikzpicture}
\begin{loglogaxis}[
  title={DNC Compute-to-Zero Loss},
  title style={font=\small, yshift=2pt},
  xlabel={Model size (parameters)},
  ylabel={forward-passes to zero loss},
  xmajorgrids=true,
  ymajorgrids=true,
  grid style=dashed,
  width=8.3cm,
  height=6.0cm,
  xtick={3e5,1e6,4e6,1.7e7,6.7e7,2.7e8,1.1e9},
  xticklabels={300k,1M,4M,17M,68M,270M,1.1B},
  legend style={
    at={(0.5,-0.28)},
    anchor=north,
    draw=none,
    font=\scriptsize,
    row sep=-2pt,
    column sep=4pt,
    cells={anchor=west},
    inner xsep=1pt,
    inner ysep=1pt,
  },
  legend columns=3,
  xlabel style={font=\small},
  ylabel style={font=\small},
]

\addplot+[mark=*, thick, color=black] table {
304357 8610
1132901 3540
4362853 1440
17114213 600
67782757 270
269783141 111
};
\addlegendentry{BPTT}

\addplot+[mark=triangle*, thick, color=blue, solid] table {
304357 74160
1132901 40224
4362853 22176
17114213 10560
67782757 8400
269783141 10752
1076437093 12944
};
\addlegendentry{1SPSA @ 8}

\addplot+[mark=square*, thick, color=red, solid] table {
304357 226560
1132901 115200
4362853 48000
17114213 34560
67782757 23040
269783141 15360
1076437093 21120
};
\addlegendentry{1SPSA @ 96}

\addplot+[mark=diamond*, thick, color=green!60!black, solid] table {
304357 153600
1132901 105472
4362853 53248
17114213 22528
67782757 16384
269783141 20480
1076437093 24576
};
\addlegendentry{1SPSA @ 512}

\addplot+[mark=triangle*, very thick, color=blue, densely dotted, mark options={solid}] table {
304357 5848
1132901 5423
4362853 5644
17114213 6171
67782757 4777
269783141 4998
1076437093 6698
};
\addlegendentry{1.5-SPSA @ 8}

\addplot+[mark=square*, very thick, color=red, densely dotted, mark options={solid}] table {
304357 18142
1132901 9843
4362853 9071
17114213 6562
67782757 5211
269783141 7527
1076437093 7720
};
\addlegendentry{1.5-SPSA @ 96}

\addplot+[mark=diamond*, very thick, color=green!60!black, densely dotted, mark options={solid}] table {
304357 82000
1132901 61500
4362853 43050
17114213 22550
67782757 16400
269783141 11275
1076437093 9225
};
\addlegendentry{1.5-SPSA @ 512}

\end{loglogaxis}
\end{tikzpicture}
\caption{
DNC overfitting stress test: compute (measured as forward-passes or equivalent) to reach a near-zero loss threshold vs model size. We convert steps to forward-passes as follows: 1SPSA uses $2n_{\text{pert}}$ forward-passes per step, 1.5-SPSA uses $(2n_{\text{pert}}+1)$ forward-passes per step, and BPTT is approximated as $3$ forward-pass equivalents per step (forward + backward).
BPTT cannot run the 1.1B model in this setup due to GPU memory limits.
}
\label{fig:overfit_15spsa_vs_1spsa_forward_passes}
\end{figure}

\subsection{Experiment Configurations}
\label{sec:exp_configs}

Unless stated otherwise, all experiments use the following shared setup. We use the above configuration for: OPT-13B, OPT-30B, Qwen3-8B, and Qwen3-1.7B.

\paragraph{Zero-order methods (1SPSA / 1.5-SPSA).}
We sweep a tied step size and perturbation size, $\lambda=\epsilon$, over
\[
\lambda=\epsilon \in \{10^{-3},\,5\cdot 10^{-4},... , 10^{-7}\}.
\]
We fix the saturation exponent to $\alpha=0.1$ for all runs as explained in \cref{tab:alpha_ablation}, and apply a plateau schedule that simultaneously cuts both $\lambda$ and $\epsilon$ when the validation loss fails to improve for 10 consecutive evaluations. We do not use any learning-rate warmup. We sweep the number of perturbations per step $n_{\text{pert}} \in \{40,60,100\}$ and the effective batch size $\in \{128,256\}$ (via gradient accumulation if needed), with sequence length $256$ unless noted otherwise. In practice, $n_{\text{pert}}=60$ and $\lambda=\epsilon=10^{-4}$ is typically sufficient for stable 1.5-SPSA performance across our settings.

\paragraph{Backpropagation baseline (Adam).}
For Adam runs, we use $(\beta_1,\beta_2)=(0.9,0.99)$ with weight decay $10^{-3}$. 

\paragraph{Hardware.}
All runs are executed on a cluster of $8\times$A100 GPUs.

\paragraph{DNC setting.}
For DNC experiments, we use the same as above except we use a character-level tokenizer, sequence length $100$, an input embedding size of $128$, and add $\lambda=\epsilon=10^{-2}$ in the sweep for these runs as the weights are xavier-initialized.

\subsection{JL preservation of directional curvature}

\label{app:jl_curvature}

This appendix formalizes the following intuition used in our method:
\emph{if a Johnson--Lindenstrauss (JL) embedding preserves geometry (distances / inner products) on a set of perturbation directions, then it also preserves the directional curvature terms that define our probe-space preconditioner, with high probability.}
Throughout, fix a point $\theta\in\R^d$ and assume $L:\R^d\to\R$ is twice differentiable in a neighborhood of $\theta$ with Hessian
$H := \nabla^2 L(\theta)\in\R^{d\times d}$ (symmetric, possibly indefinite).

\begin{lemma}[JL inner-product preservation (normalized, pairwise form)]
\label{lem:jl_inner_product}
Let $R\in\R^{m\times d}$ be a \emph{linear} map (as in standard JL transforms: Gaussian, subgaussian, SRHT, CountSketch, etc.).
Fix a finite collection of pairs $\{(x_i,y_i)\}_{i=1}^n$ with $x_i\neq 0$ and $y_i\neq 0$, and define unit vectors
$u_i := x_i/\|x_i\|$ and $v_i := y_i/\|y_i\|$.
Assume the JL norm guarantee holds simultaneously on the finite set
\[
\mathcal{S}
:= \{u_i,\ v_i,\ u_i+v_i,\ u_i-v_i\}_{i=1}^n,
\]
namely for some $\varepsilon_{\mathrm{JL}}\in(0,1)$,
\begin{equation}
(1-\varepsilon_{\mathrm{JL}})\|s\|^2 \le \|Rs\|^2 \le (1+\varepsilon_{\mathrm{JL}})\|s\|^2
\qquad \forall s\in \mathcal{S}.
\label{eq:jl_norm_S}
\end{equation}
Then for every $i\in\{1,\dots,n\}$,
\begin{equation}
\big|\langle R x_i,\ R y_i\rangle - \langle x_i,\ y_i\rangle\big|
\;\le\;
\varepsilon_{\mathrm{JL}}\,\|x_i\|\,\|y_i\|.
\label{eq:jl_ip}
\end{equation}
\end{lemma}

\begin{proof}
For each $i$, apply polarization to the unit vectors $u_i,v_i$:
\[
\langle Ru_i,Rv_i\rangle-\langle u_i,v_i\rangle
\]
\[
\begin{aligned}
=\tfrac14\Big(&\big[\|R(u_i+v_i)\|^2-\|u_i+v_i\|^2\big]\\
&-\big[\|R(u_i-v_i)\|^2-\|u_i-v_i\|^2\big]\Big).
\end{aligned}
\]
Using \eqref{eq:jl_norm_S}, we have
$\big|\|R(u_i\pm v_i)\|^2-\|u_i\pm v_i\|^2\big|\le \varepsilon_{\mathrm{JL}}\|u_i\pm v_i\|^2$,
hence
\begin{equation}
    \begin{aligned}
\big|\langle Ru_i,Rv_i\rangle-\langle u_i,v_i\rangle\big|
&\le \tfrac{\varepsilon_{\mathrm{JL}}}{4}\big(\|u_i+v_i\|^2+\|u_i-v_i\|^2\big)
\\&=\tfrac{\varepsilon_{\mathrm{JL}}}{4}\cdot 4
\\&=\varepsilon_{\mathrm{JL}}.
\end{aligned}
\end{equation}
Multiplying by $\|x_i\|\|y_i\|$ yields \eqref{eq:jl_ip}.
\end{proof}

\paragraph{Remark (always-valid weaker bound without normalization).}
If instead you assume JL norm preservation on $\{x_i,y_i,x_i\pm y_i\}$, then polarization yields the bound
\[
\big|\langle R x_i,\ R y_i\rangle - \langle x_i,\ y_i\rangle\big|
\;\le\;
\tfrac{\varepsilon_{\mathrm{JL}}}{2}\big(\|x_i\|^2+\|y_i\|^2\big),
\]
which is weaker than \eqref{eq:jl_ip} when $\|x_i\|$ and $\|y_i\|$ are very different.

\begin{theorem}[JL preservation of directional curvature on a finite set]
\label{thm:jl_curvature}
Fix $\theta\in\R^d$ and a twice-differentiable $L$ near $\theta$ with Hessian $H:=\nabla^2 L(\theta)\in\R^{d\times d}$.
Fix directions $\{z_i\}_{i=1}^n$ and define $v_i := Hz_i$ and the (true) directional curvature
\[
c_i := z_i^\top H z_i = \langle z_i,\ v_i\rangle.
\]
Let $R\in\R^{m\times d}$ be a (linear) JL transform (e.g., i.i.d.\ Gaussian entries $R_{jk}\sim \mathcal{N}(0,1/m)$).
Define the JL curvature proxy
\[
\tilde c_i := \langle Rz_i,\ Rv_i\rangle = \langle Rz_i,\ R(Hz_i)\rangle.
\]
Let $\varepsilon_{\mathrm{JL}}\in(0,1)$ and $\delta\in(0,1)$. If
\begin{equation}
m \;\gtrsim\; \varepsilon_{\mathrm{JL}}^{-2}\log\frac{n}{\delta},
\label{eq:m_scaling}
\end{equation}
then with probability at least $1-\delta$ over $R$, simultaneously for all $i\in\{1,\dots,n\}$,
\begin{equation}
\big|\tilde c_i - c_i\big|
\;\le\;
\varepsilon_{\mathrm{JL}}\,\|z_i\|\,\|Hz_i\|.
\label{eq:curv_abs_bound}
\end{equation}
Moreover, whenever $c_i\neq 0$,
\begin{equation}
\frac{|\tilde c_i - c_i|}{|c_i|}
\;\le\;
\varepsilon_{\mathrm{JL}}\cdot \frac{\|z_i\|\,\|Hz_i\|}{|\langle z_i,Hz_i\rangle|}
\;=\;
\frac{\varepsilon_{\mathrm{JL}}}{|\cos\angle(z_i,Hz_i)|}.
\label{eq:curv_rel_bound}
\end{equation}
\end{theorem}

\begin{proof}
For each $i$ with $v_i=Hz_i\neq 0$, apply Lemma~\ref{lem:jl_inner_product} to the pair $(x_i,y_i)=(z_i,v_i)$.
This requires JL norm preservation on the set
$\{z_i/\|z_i\|,\ v_i/\|v_i\|,\ z_i/\|z_i\|\pm v_i/\|v_i\|\}_{i=1}^n$,
whose cardinality is at most $4n$.
Standard JL finite-set bounds imply that \eqref{eq:jl_norm_S} holds for all elements of this set
with probability at least $1-\delta$ provided $m\gtrsim \varepsilon_{\mathrm{JL}}^{-2}\log((4n)/\delta)$,
which is equivalent to \eqref{eq:m_scaling} up to constants.
Under this event, Lemma~\ref{lem:jl_inner_product} yields
$|\langle Rz_i,Rv_i\rangle-\langle z_i,v_i\rangle|\le \varepsilon_{\mathrm{JL}}\|z_i\|\|v_i\|$,
which is exactly \eqref{eq:curv_abs_bound} (since $v_i=Hz_i$).
The relative form \eqref{eq:curv_rel_bound} follows by dividing by $|c_i|=|\langle z_i,v_i\rangle|$ and using
$|\langle z_i,v_i\rangle|=\|z_i\|\,\|v_i\|\,|\cos\angle(z_i,v_i)|$.
If $v_i=0$, then $c_i=\tilde c_i=0$ and the bounds hold trivially.
\end{proof}

\paragraph{Remarks.}
\begin{itemize}
\item \textbf{Linearity of $R$.} The definition $\tilde c_i=\langle Rz_i,R(Hz_i)\rangle$ uses only that $R$ is linear so that $R(Hz_i)$ is well-defined as applying the same JL map to the vector $Hz_i$. This is standard for JL embeddings.
\item \textbf{Relationship between $m$ and $d$.} The guarantee is \emph{finite-set}: $m$ scales like $\log n$ (not $d$) because we only need geometry preservation on the particular vectors used in the step. Note that $Hz_i$ is formed in the ambient $\R^d$ first (conceptually, in analysis), and then $R$ is applied to $z_i$ and $Hz_i$ as vectors in $\R^d$. No claim is made that one can apply $H$ in the compressed space.
\item \textbf{Small-curvature regime.} The relative bound depends on $|\cos\angle(z_i,Hz_i)|$ and becomes loose when $z_i$ lies near the null space of $H$ (or, more generally, when $z_i$ and $Hz_i$ are nearly orthogonal). This is exactly why curvature-based weights must regularize near $c_i\approx 0$ (e.g., via a floor $\lambda_{\mathrm{reg}}$ and saturation exponent $\alpha$).
\end{itemize}

\begin{lemma}[Bias of the 3-point curvature estimator (correct Lagrange remainder)]
\label{lem:fd_curv_bias}
Assume $L$ is three-times continuously differentiable on a neighborhood of the segment
$\{\theta+t z:\ |t|\le \epsilon_{\mathrm{fd}}\}$.
Define the 3-point estimator
\[
\hat c(z)
:=
\frac{L(\theta+\epsilon_{\mathrm{fd}} z)-2L(\theta)+L(\theta-\epsilon_{\mathrm{fd}} z)}{\epsilon_{\mathrm{fd}}^2}.
\]
Let $H=\nabla^2L(\theta)$. Then there exist $\xi_+\in(0,1)$ and $\xi_-\in(-1,0)$ such that
\begin{equation}
\begin{aligned}
\hat c(z)
=
z^\top H z
\;+ \;
\frac{\epsilon_{\mathrm{fd}}}{6}
\Big(
\nabla^3 L(\theta+\xi_+\epsilon_{\mathrm{fd}} z)[z,z,z]\\
-
\nabla^3 L(\theta+\xi_-\epsilon_{\mathrm{fd}} z)[z,z,z]
\Big).
\label{eq:fd_bias_exact}
\end{aligned}
\end{equation}
Consequently,
\begin{equation}
|\hat c(z)-z^\top H z|
\;\le\;
\frac{\epsilon_{\mathrm{fd}}}{3}\,
\sup_{|t|\le \epsilon_{\mathrm{fd}}}
\big|\nabla^3 L(\theta+t z)[z,z,z]\big|.
\label{eq:fd_bias_bound_general}
\end{equation}
In particular, if $|\nabla^3 L(u)[z,z,z]|\le M\|z\|^3$ along the segment, then
\begin{equation}
|\hat c(z)-c(z)| \;\le\; \frac{\epsilon_{\mathrm{fd}}}{3}\, M \|z\|^3.
\label{eq:fd_bias_bound}
\end{equation}
\end{lemma}

\begin{proof}
Let $f(t)=L(\theta+t z)$. Taylor's theorem about $t=0$ to second order with Lagrange remainder gives
\[
f(\pm \epsilon_{\mathrm{fd}})
=
f(0)\pm f'(0)\epsilon_{\mathrm{fd}}+\tfrac12 f''(0)\epsilon_{\mathrm{fd}}^2
\pm \tfrac16 f^{(3)}(\eta_\pm)\epsilon_{\mathrm{fd}}^3
\]
for some $\eta_+\in(0,\epsilon_{\mathrm{fd}})$ and $\eta_-\in(-\epsilon_{\mathrm{fd}},0)$.
Plugging into the symmetric stencil cancels the odd first-order term and yields
\[
\hat c(z)=f''(0)+\frac{\epsilon_{\mathrm{fd}}}{6}\big(f^{(3)}(\eta_+)-f^{(3)}(\eta_-)\big).
\]
Since $f''(0)=z^\top \nabla^2 L(\theta) z$ and $f^{(3)}(\eta)=\nabla^3 L(\theta+\eta z)[z,z,z]$,
we obtain \eqref{eq:fd_bias_exact} by setting $\eta_\pm=\xi_\pm\epsilon_{\mathrm{fd}}$.
The bounds \eqref{eq:fd_bias_bound_general}--\eqref{eq:fd_bias_bound} follow immediately by taking absolute values.
\end{proof}

\paragraph{Notation note.}
To avoid overloading symbols, this appendix uses $\varepsilon_{\mathrm{JL}}$ for the JL distortion level and
$\epsilon_{\mathrm{fd}}$ for the finite-difference probe radius (the SPSA step used in $\hat c(\cdot)$).

\paragraph{Combined guarantee (one-line takeaway).}
For the finite set $\{z_i\}_{i=1}^n$ used in a step, with probability at least $1-\delta$ over the JL map $R$,
simultaneously for all $i$,
\begin{equation}
\begin{aligned}
|\langle Rz_i,\ R(Hz_i)\rangle - \hat c(z_i)|
\;\le\;
\underbrace{\varepsilon_{\mathrm{JL}}\,\|z_i\|\,\|Hz_i\|}_{\text{JL inner-product distortion}}
\; \\ + \;
\underbrace{\frac{\epsilon_{\mathrm{fd}}}{3}\sup_{|t|\le \epsilon_{\mathrm{fd}}}
|\nabla^3 L(\theta+t z_i)[z_i,z_i,z_i]|}_{\text{finite-difference bias}}.
\end{aligned}
\end{equation}

\subsection{E. $\alpha$ sweep and failure-rate plots}
\label{app:alpha_sweep}

\begin{table}[H]
\centering
\small
\begin{tabular}{lcccc}
\toprule
\textbf{Alpha} & \textbf{Iter} & \textbf{Best Test} & \textbf{Note} \\
\midrule
0.01  & 51 & 93.6\% &  $\times$ DIVERGED \\
 \textbf{0.1}   & \textbf{50} & \textbf{94.5\%} & \textbf{Stable} \\
0.25  & 51 & 91.7\% & Stable \\
0.5   & 50 & 89.7\% & Stable \\
0.75  & 50 & 89.0\% & Stable \\
1.0   & 52 & 89.2\% & Stable \\
\bottomrule
\end{tabular}
\caption{Alpha ablation: stability and performance. Lower alpha (0.01) risks collapse; moderate alpha (0.1) creates stability without sacrificing peak accuracy. }
\label{tab:alpha_ablation}
\end{table}

\subsection{E. $\epsilon$ sweep and failure-rate plots}
\begin{table}[H]
\centering
\small
\begin{tabular}{lccll}
\toprule
\textbf{Epsilon} & \textbf{Best Test} & \textbf{Status} \\
\midrule
1e-2 & 80.5\%  & Never improved \\
5e-3 & 86.5\%  & Improved slightly then diverged \\
1e-3 & 93.0\%  & Good \\
5e-4 & 93.6\%  & Very good \\
 \textbf{1e-4} & \textbf{94.5\%}  &  \textbf{BEST (eps=lr)} \\
5e-5 & 92.7\%  & Good \\
1e-5 & 80.5\%  & Never improved \\
5e-6 & 80.5\%  & Never improved \\
\bottomrule
\end{tabular}
\caption{Epsilon Ablation Results ($\text{lr}=10^{-4}$ for all runs). Best performance at $\epsilon=10^{-4}$. The range $10^{-5} > \epsilon > 10^{-3}$ converge well, outside this range they fail. Intuitively, this makes sense as we are measuring where we are stepping to vs. measuring farther or earlier than we are stepping to. }
\label{tab:eps_ablation}
\end{table}

\onecolumn
\subsection{F. Systems appendix: Bit packing + Triton kernels}
\label{app:triton_code_bitpack}
\subsection{Wall-Clock Per Probe Generation and Step Comparisons on OPT-13B and 96 perturbations per step}

\begin{table}[H]
\centering
\small
\setlength{\tabcolsep}{6pt}
\begin{tabular}{lrrrrr}
\toprule
Method & \multicolumn{1}{c}{Probe gen} & \multicolumn{1}{c}{96 perts} & \multicolumn{1}{c}{Step} & \multicolumn{1}{c}{Speedup} & \multicolumn{1}{c}{Speedup} \\
& \multicolumn{1}{c}{(ms/pert)} & \multicolumn{1}{c}{(s)} & \multicolumn{1}{c}{(s)} & \multicolumn{1}{c}{(probe gen)} & \multicolumn{1}{c}{(step)} \\
\midrule
\texttt{triton\_bitpacked}  & 102.1 & 9.8  & 21.7 & 4.9$\times$ & 2.76$\times$ \\
\texttt{pytorch\_flat}      & 246.6 & 23.7 & 35.6 & 2.0$\times$ & 1.68$\times$ \\
\texttt{pytorch} & 499.5 & 48.0 & 59.9 & 1.0$\times$ & 1.00$\times$ \\
\bottomrule
\end{tabular}
\caption{
Perturbation generation comparison. ``Speedup (probe gen)'' is computed from ms/pert relative to naive pytorch implementation. ``Speedup (total)'' is computed from end-to-end total time relative to original, per parameter pytorch implementation. Analysis done in series on a single A100 to isolate the impact of just the triton kernel. 
}
\label{tab:probe_gen_speed}
\end{table}

\noindent\textbf{Reference Triton code}
\begingroup\small
\begin{verbatim}
@triton.jit
def _unpack_and_apply(w_ptr, packed_ptr, n_elements, alpha, BLOCK_SIZE: tl.constexpr):
    pid = tl.program_id(axis=0)
    offsets = pid * BLOCK_SIZE + tl.arange(0, BLOCK_SIZE)
    mask = offsets < n_elements
    w = tl.load(w_ptr + offsets, mask=mask)
    byte_idx = offsets // 8
    bit_idx = offsets % 8
    packed_byte = tl.load(packed_ptr + byte_idx, mask=mask)
    bit = (packed_byte >> bit_idx) & 1
    sign = tl.where(bit == 1, 1.0, -1.0)
    tl.store(w_ptr + offsets, w + alpha * sign, mask=mask)

@triton.jit
def _unpack_and_accumulate(
    grad_ptr, packed_ptr, n_elements, coeff, BLOCK_SIZE: tl.constexpr
):
    pid = tl.program_id(axis=0)
    offsets = pid * BLOCK_SIZE + tl.arange(0, BLOCK_SIZE)
    mask = offsets < n_elements
    byte_idx = offsets // 8
    bit_idx = offsets % 8
    packed_byte = tl.load(packed_ptr + byte_idx, mask=mask)
    bit = (packed_byte >> bit_idx) & 1
    sign = tl.where(bit == 1, 1.0, -1.0)
    grad = tl.load(grad_ptr + offsets, mask=mask)
    tl.store(grad_ptr + offsets, grad + coeff * sign, mask=mask)
\end{verbatim}
\endgroup

\end{document}